\documentclass{article}

     \PassOptionsToPackage{numbers, compress}{natbib}
\usepackage[preprint]{neurips_2025}

\usepackage[utf8]{inputenc} %
\usepackage[T1]{fontenc}    %
\usepackage{url}            %
\usepackage{booktabs}       %
\usepackage{amsfonts}       %
\usepackage{nicefrac}       %
\usepackage{microtype}      %
\usepackage{wrapfig}

\usepackage{mmll}
\usepackage[colorlinks,
            linkcolor=red,
            anchorcolor=blue,
            citecolor=blue, 
            pagebackref=true
           ]{hyperref}   
\usepackage[nameinlink]{cleveref}
\usepackage{subcaption}
\usepackage{enumitem}

\newcommand{\pan}[1]{}

\newcommand{\algname}{\text{R}^2\text{CSL}}
\crefname{assumption}{assumption}{assumptions}
\allowdisplaybreaks

\title{How to Provably Improve Return Conditioned Supervised Learning?}

\author{%
    Zhishuai Liu$^1$, Yu Yang$^1$, Ruhan Wang$^2$, Pan Xu$^1$, Dongruo Zhou$^2$\\
    $^1$Duke University\\
    $^2$Indiana University Bloomington\\
    \texttt{\{zhishuai.liu,yu.yang,pan.xu\}@duke.edu,
    \{ruhwang, dz13\}@iu.edu}
}

\begin{document}

\maketitle

\begin{abstract}
In sequential decision-making problems, Return-Conditioned Supervised Learning (RCSL) has gained increasing recognition for its simplicity and stability in modern decision-making tasks. Unlike traditional offline reinforcement learning (RL) algorithms, RCSL frames policy learning  as a supervised learning problem by taking both the state and return as input. This approach eliminates the instability often associated with temporal difference (TD) learning in offline RL. However, RCSL has been criticized for lacking the stitching property, meaning its performance is inherently limited by the quality of the policy used to generate the offline dataset. To address this limitation, we propose a principled and simple framework called Reinforced RCSL. The key innovation of our framework is the introduction of a concept we call the  in-distribution optimal return-to-go. This mechanism leverages our policy to identify the best achievable in-dataset future return based on the current state, avoiding the need for complex return augmentation techniques. Our theoretical analysis demonstrates that Reinforced RCSL can consistently outperform the standard RCSL approach. Empirical results further validate our claims, showing significant performance improvements across a range of benchmarks.
\end{abstract}

\section{Introduction}

Reinforcement Learning (RL) has emerged as a powerful paradigm for decision-making and sequential learning \citep{sutton2018reinforcement}, achieving remarkable successes across domains such as robotics \citep{kober2013reinforcement, singh2022reinforcement}, healthcare \citep{liu2017deep, yu2021reinforcement, liu2023deep}, games \citep{mnih2015human,vinyals2019grandmaster}, and training large language models \citep{guo2025deepseek}. Among the various RL paradigms, offline RL \citep{levine2020offline, jin2021pessimism, fujimoto2021minimalist} has gained substantial attention due to its ability to learn policies from pre-collected datasets without requiring interaction with the environment, which is particularly appealing in scenarios where exploration is costly, unsafe, or impractical. A key advantage of offline RL lies in its ability to leverage data generated by many existing policies, enabling the discovery of robust and effective behavior patterns.

Within offline RL, return-conditioned supervised learning (RCSL) has recently attracted significant traction \citep{kostrikov2022offline, kumar2019reward, chen2021decision, emmons2022rvs, wang2024return}. RCSL reframes the policy learning problem as a supervised learning problem: the input consists of the state and the return from the current state, while the output is the optimal action for that state. Compared to classical offline RL algorithms, which primarily rely on dynamic programming (DP) approaches \citep{kumar2020conservative, kostrikov2022offline}, RCSL is easier to train, more straightforward to tune, and often achieves competitive performance across a variety of tasks. Notable RCSL methods include Decision Transformer (DT) \citep{chen2021decision} and Reinforcement Learning via Supervised Learning (RVS) \citep{emmons2022rvs}. However, a critical limitation of RCSL lies in its lack of stitching ability—that is, its inability to derive a policy that exceeds the performance of the policies used to generate the offline dataset. This limitation arises because RCSL tends to follow trajectories from the dataset without effectively combining the best parts of different trajectories, which is essential for achieving superior performance. This limitation has been rigorously analyzed and demonstrated in \cite{brandfonbrener2022does}, suggesting that RCSL's lack of stitching ability may be an inherent property of the algorithm design.

Recently, several works have taken initial steps toward addressing this limitation and exploring potential solutions. For example, \citet{yamagata2023q} proposed the QDT method, which relabels returns using a pre-learned optimal Q-function; \citet{wu2024elastic} introduced the Elastic Decision Transformer, which dynamically adjusts the input sequence length; and \citet{zhuang2024reinformer} developed the Reinformer, which incorporates expectile regression into the Decision Transformer framework. These works have empirically demonstrated some level of stitching ability, indicating that enhancing stitching ability within RCSL methods is indeed possible. However, the theoretical understanding of RCSL remains underdeveloped compared to dynamic programming-based RL methods, where theoretical guarantees are more mature and extensively studied. Bridging this gap between empirical success and theoretical guarantees remains an open challenge.
In this context, we pose the following critical question:
\begin{center} 
\textbf{Can we improve return-conditioned supervised learning to have provable stitching ability?} 
\end{center}

Our main contributions are listed as follows. 
\begin{itemize}[leftmargin = *]%
    \item We introduce \textbf{reinforced RCSL ($\algname$)}, an advancement over RCSL that strictly improves its performance. The key innovation of $\algname$ is a new concept called the \emph{in-distribution optimal return-to-go (RTG)}, which characterizes the highest accumulated reward a RCSL method can achieve under the offline trajectory distribution. This quantity can be directly learned via supervised learning, thereby avoiding the need for dynamic programming, which is commonly used in classical offline RL methods \citep{kumar2020conservative}. We demonstrate that $\algname$, by incorporating the in-distribution optimal RTG, learns an \emph{in-distribution optimal stitched policy} that surpasses the best policy achievable by traditional RCSL methods. To the best of our knowledge, this is the first work to \textbf{provably surpass RCSL without dynamic programming}.
    
    \item We provide a \textbf{sample complexity analysis} for various environments, including tabular MDPs and MDPs with general function approximation. We show that the sample complexity of $\algname$ to achieve the in-distribution optimal stitched policy is of the same order as classical RCSL methods \citep{brandfonbrener2022does}, while $\algname$ converges to a superior policy. We further propose two realizations of $\algname$ using \emph{expectile regression} \citep{newey1987asymmetric} and \emph{quantile regression} \citep{koenker2001quantile} for the in-distribution optimal RTG estimation and analyze their theoretical guarantees.

    \item We conduct comprehensive experimental studies  under a simulated point mass environment, the D4RL gym and Antmaze environments to showcase the effectiveness of the $\algname$ algorithm. Experiment results demonstrate that (1) $\algname$ achieves the stitching ability, and outperforms RCSL-type algorithms like RvS and DT; (2) The $\algname$ framework is also flexible enough to incorporate dynamic programming components, and it achieves performance comparable to the state-of-the-art QT algorithm \citep{hu2024q}, while maintaining its simplicity.
    
    \item To further extend the notion of in-distribution optimal RTG, we study the \emph{multi-step in-distribution optimal RTG}, which generalizes the original quantity. We prove that by increasing the number of steps considered in the in-distribution optimal RTG, $\algname$ is capable of finding the \emph{optimal in-distribution policy}, a result that was previously only achievable by dynamic programming-based algorithms \citep{kumar2020conservative}. Our findings \textbf{close a long-standing theoretical gap} between RCSL-type methods and dynamic programming-based approaches.

\end{itemize}

\section{Related Work}
\paragraph{Empirical studies about RCSL.} 
Conditional sequence modeling \citep{srivastava2019training, janner2021offline,schmidhuber2019reinforcement} has emerged as a promising approach to solving offline reinforcement learning through supervised learning. This paradigm learns from the offline dataset a behavior policy with state and return-to-go (RTG) as input, and then predicts subsequent actions by conditioning on the current state and a conditioning function that encodes specific metrics for future trajectories. Existing methods within the RCSL framework have demonstrated significant empirical success. In particular, DT \citep{chen2021decision, furuta2022generalized, zheng2022online} and RvS \citep{emmons2022rvs}
adopt the vanilla RTG as the conditioning function to predict the optimal action. 
However, vanilla RCSL methods, such as DT and RvS, lack stitching ability \citep{brandfonbrener2022does}, a crucial property of dynamic programming-based methods like Q-learning \citep{watkins1992q} and TD \citep{sutton2018reinforcement, tesauro1995temporal}. To address this limitation, \citet{yamagata2023q} proposed QDT, which enhances RCSL by relabeling RTGs in the dataset using a pre-trained optimal Q-function and then training a DT on the relabeled data. \citet{wu2024elastic} introduced the Elastic Decision Transformer,  which enables trajectory stitching during action inference by adjusting the history length used in DT to discard irrelevant or suboptimal past experiences. 
\citet{zhuang2024reinformer} proposed Reinformer, which incorporates an expectile regression model to estimate the in-distribution optimal RTG as the conditioning function. 
\citet{yangdichotomy,paster2022you} studied the limitations of RCSL in stochastic environments. \citet{gao2024act,xu2022policy} propose DT-based and offline RL methods that utilize expectile regression
Despite the empirical successes of these methods, none of them provide any theoretical guarantees on the stitching ability.

\paragraph{Theoretical studies of RCSL.}
Compared to the extensive theoretical studies on dynamic programming (DP)-based algorithms, the theoretical analysis of RCSL methods remains relatively limited. From a theoretical perspective, \citet{brandfonbrener2022does} examined the finite-sample guarantees of RCSL methods, including DT and RvS, and identified the fundamental challenge of sample complexity due to the need for sufficient return and state coverage. \citet{zheng2024how} investigated the goal-conditioned supervised learning (GCSL) setting and provided a regret analysis for a goal relabeling method. \citet{zhu2024provably} also studied the GCSL setting, offering a finite-sample analysis for GCSL with $f$-divergence regularization. Our work falls within this line of research on RCSL, but we \emph{rigorously establish} the sample complexity of $\algname$, demonstrating that it achieves a policy superior to that of a classical RCSL method.

\section{Problem Formulation}
\label{sec: Preliminary}

\paragraph{Reinforcement learning.} 

We consider episodic Markov Decision Processes (MDPs) in this work, where each MDP is represented by a tuple $(\mathcal{S}, \mathcal{A}, P, r, H, \rho)$. Here, $\mathcal{S}$ denotes the state space, $\mathcal{A}$ is the action space, $P(s' | s, a)$ specifies the probability of transitioning to state $s'$ after taking action $a$ at state $s$, and the reward function $|r_h(s, a)| \leq 1$ assigns a reward to taking action $a$ at state $s$. The horizon length $H$ indicates the fixed number of steps in each episode, and $\rho(s)$ defines the initial state distribution, giving the probability of starting with state $s$.

The agent interacts with the environment with a policy $\pi:\cS\times [H]\rightarrow \Delta(\cA)$ in a policy class $\Pi$. At each timestep $h$, the agent observes the current state $s_h \in \mathcal{S}$, selects an action $a_h \in \mathcal{A}$ according to its policy $\pi_h$, and transitions to the next state $s_{h+1}$, which is sampled from the transition dynamics $P(\cdot | s_h, a_h)$. Simultaneously, the agent receives a reward $r_h = r_h(s_h, a_h)\in[0,1]$. The objective of the agent is to learn a policy $\pi: \mathcal{S} \to \mathcal{A}$ that maximizes the expected cumulative reward over an episode: 
$
\mathbb{E}_{\pi}^P \big[ \sum_{h=1}^H r(s_h, a_h)\big],
$
where the expectation is taken over the randomness in the initial state distribution $\rho$, the policy $\pi$, and the transition dynamics $P$. The value function $V_h^\pi(s) = \mathbb{E}_\pi \big[ \sum_{t=h}^H r_t(s_t, a_t) \mid s_h = s \big]$ represents the expected cumulative reward starting from state $s$ at timestep $h$ and following policy $\pi$ thereafter. 
The Q-function is given by:
$
Q_h^\pi(s, a) = r_h(s, a) + \mathbb{E}_{s' \sim P_h(\cdot \mid s, a)} \big[ V_{h+1}^\pi(s') \big].
$
The optimal value function $V_h^*(s)$ and the optimal Q-function $Q_h^*(s, a)$ are defined similarly but correspond to the optimal policy $\pi^*$, which maximizes the expected cumulative reward. Finally, we define $J(\pi)=\EE_{s_1 \sim \rho}V_1^\pi(s_1)$.

\paragraph{The offline setting.} 
We consider an offline dataset $\cD$ collected by a behavior policy $\beta$, where the dataset size is $|\cD| = N$. $\cD$ consists of trajectories: $\cD = \{\tau^k \}_{k=1}^N$, with each trajectory $\tau^k$ represented as 
$\tau^k = (s_1^k, a_1^k, r_1^k, \cdots, s_H^k, a_H^k, r_H^k),$
where $s_1^k \sim \rho$ is the initial state drawn from the initial state distribution $\rho$, $a_h^k \sim \beta(\cdot|s_h^k)$ is the action selected by the behavior policy $\beta$ at step $h$, and $s_{h+1}^k \sim P(\cdot | s_h^k, a_h^k)$ is the next state sampled from the transition dynamics $P$. We use $d^{\beta}_h(s)$ to denote the state distribution of state $s$ under the behavior policy $\beta$ at step $h$. Our goal is to leverage this offline dataset $\cD$ to learn an effective policy that performs well in the underlying environment.

\paragraph{The RCSL framework.}%

The RCSL framework aims to learn a policy by modeling the distribution of actions conditioned on the state, the stage and the return of the trajectory, denoted by $\pi:\cS \times[H]\times \RR\rightarrow \Delta(\cA)$. Specifically, RCSL algorithms optimize the policy by minimizing the empirical negative log-likelihood loss over the offline dataset $\cD$: $\hat\pi = \argmin_{\pi \in \Pi}\hat{L}(\pi)$, where
\begin{align}
    \textstyle \hat{L}(\pi) = -\sum_{\tau \in \cD} \sum_{t=1}^H \log \pi(a_t \mid s_t,t, g(\tau,t)),\label{def:mle}
\end{align}
where $g(\tau, h) = \sum_{t=h}^H r_t$ computes the  return-to-go (RTG, \citep{chen2021decision}) along $\tau$, starting from step $h$. This optimization aligns the learned policy $\pi$ with the observed behavior in the offline dataset, incorporating both the states and return-based context.
At test time, the RCSL algorithms utilize the learned policy $\hat{\pi}$ along with a test-time conditioning function $f(s,h)$, which determines the desired return  to set the condition of the policy. The resulting test-time policy $\pi_f$ is then defined as $\pi_f(a | s,h) := \hat \pi(a | s, f(s,h))$, where $\pi_f$ produces actions conditioned on the current state $s$ and the test-time return determined by $f(s,h)$. This design enables flexible adaptation of the policy to different test-time objectives by modifying the conditioning function $f(s,h)$.

\section{The Reinforced RCSL}

In this section, we introduce our algorithm, $\algname$, discussing its core intuition and the reasons it should be preferred. To provide context, we first revisit existing analyses of the failure modes encountered by classical RCSL, highlighting the limitations that motivate the design of $\algname$.

\paragraph{Why does RCSL fail to stitch?}
\citet{brandfonbrener2022does} showed that the policy returned by the 
\begin{wrapfigure}{r}{0.35\textwidth}
    \centering
    \includegraphics[width=\linewidth]{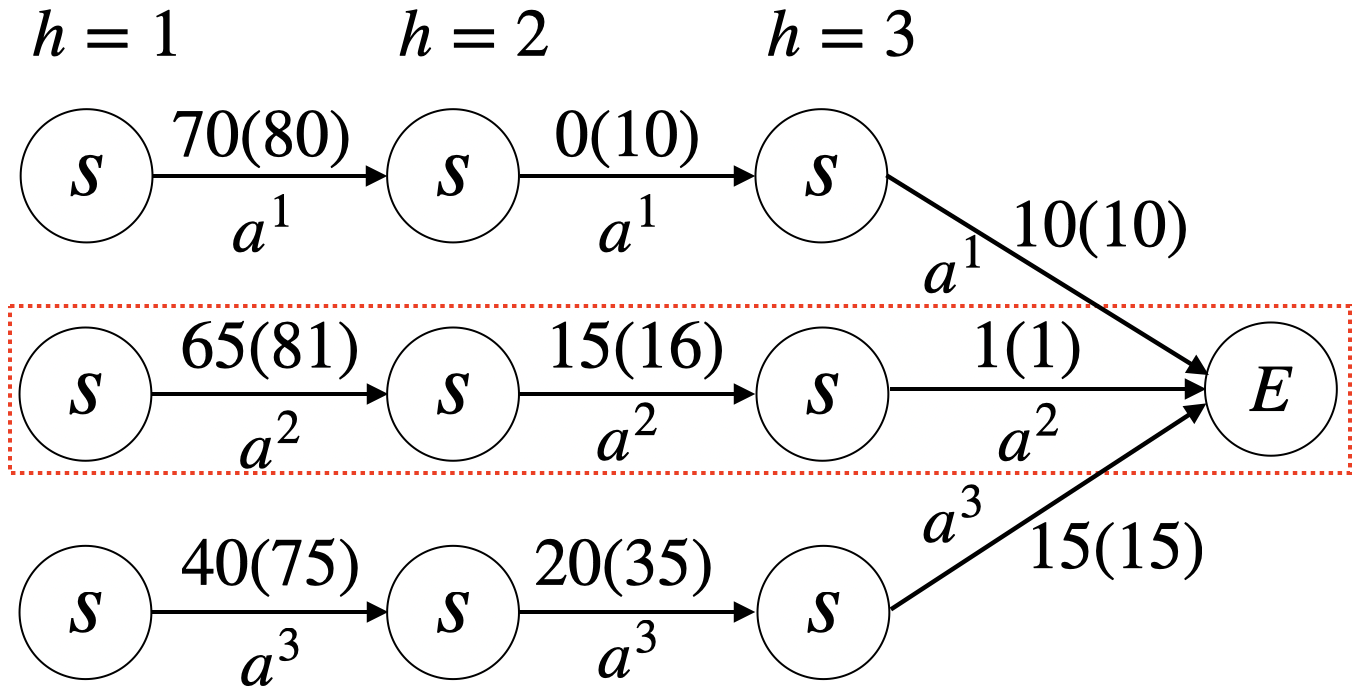}
    \caption{An example when RCSL fails to stitch.}
    \label{fig: fail-to-stitch}
\end{wrapfigure}
default RCSL framework is unable to outperform the behavior policy used to generate the dataset $\cD$. To illustrate this, given a pretrained RCSL policy $\hat{\pi}$, classical RCSL approaches such as Decision Transformer (DT) \citep{chen2021decision} and Return-Conditioned Supervision (RVS) \citep{emmons2022rvs} rely on a conditioning return function $f$ that satisfies the following conditions \citep{brandfonbrener2022does}:
\begin{itemize}[leftmargin=*]
    \item \textbf{In-distribution condition:} The initial return $f(s_1, 1)$ must be achievable with non-zero probability under the behavior policy $\beta$. 
    \item \textbf{Consistency condition:} For any trajectory, the return function must satisfy $f(s_h, h) = f(s_{h+1}, h+1) + r_h(s_h, a_h)$, ensuring consistency across steps.
\end{itemize}
These conditions are crucial for 
ensuring that, at each step $h$, there is no out-of-distribution (OOD) issue with the inputs to the policy $\pi_f$. While these conditions guarantee the validity of the learned policy, they significantly restrict the range of feasible conditioning functions $f$. Specifically, this implies that the value domain of $f$ can only be selected as $g(\tau, h)$, where $\tau$ represents any trajectory that could appear in the dataset $\cD$ generated under the behavior policy $\beta$. Therefore, the return of trajectories generated by $\pi_f$ will also be upper bounded by $g(\tau,1)$, which performs no better than the best trajectory in the dataset $\cD$. 
For illustration, we provide a toy example in \Cref{fig: fail-to-stitch} to show RCSL fails to stitch. In this example, $\cA=\{a^1,a^2,a^3\}$, $H=3$ and $\cS = \{s\}$. The state is unique and remains unchanged across stages. Each row represents a trajectory. The number outside (inside) the parentheses are rewards (return-to-go). For RCSL, we can only simply choose $f$ to be 80, 81 or 75 as the conditioning return at the initial stage, and the `optimal' RCSL policy would choose $a^2$ at each stages, i.e., the trajectory in the red box. However, a better trajectory can be obtained by stitching, e.g. taking the action sequence $(a^2, a^3, a^3)$.

\subsection{Algorithm Description}\label{sec:basicalg}
To address the non-stitching issue discussed above, we introduce $\algname$.  $\algname$ discards the consistency condition, thereby allowing greater flexibility in the selection of the conditioning function $f$. To achieve the stitching ability, at each time step $h$, $\algname$ looks ahead to search for the in-distribution optimal RTG as its conditioning function, instead of following the original return-to-go in trajectory $\tau$.
This increased flexibility and `optimal conditioning' enable $\algname$ to achieve superior results by overcoming the limitations of classical RCSL approaches. 

Formally, we begin by introducing the \emph{feasible set of trajectories} under the behavior policy $\beta$, 
$T_{\beta} := \{\tau \mid P_{\beta}(\tau) > 0\}$,
where $P_{\beta}$ is the trajectory distribution induced by $\beta$ under the transition dynamics $P$. Based on this, we can define the {\it feasible set of conditioning functions}:  $\forall \tau \in T_{\beta}$ and  $(s_h, g_h) \in \tau$, there must exist a conditioning function $f$ such that $f(s_h, h) = g_h$. 
The feasible conditioning function set $\cF_{\beta}$ is defined as:
\begin{align*}
    \cF_{\beta} &:= \{f : \cS \times [H] \rightarrow [0,H] \mid \forall (s, h) \in \text{dom}(f),  \exists \, \tau \in T_{\beta} \, \text{s.t.} \, s_h = s \, \text{and} \, f(s, h) = g_h \},
\end{align*}
where $\text{dom}(f)$ is the domain of $f$. Notably, the conditioning functions in $\cF_{\beta}$ are not constrained by the consistency assumption. At any stage $h \in [H]$, {\it the feasible set of states} is defined as $\cS_h^{\beta} := \{s \in \cS \mid d_h^{\beta}(s) > 0\}$. For any feasible state $s \in \cS_h^{\beta}$ at stage $h$, the {\it local feasible conditioning function set} is defined as $\cF_{\beta}(s, h) := \{f : \cS \times [H] \rightarrow [0,H]  \mid f \in \cF_{\beta} \, \text{and} \, (s, h) \in \text{dom}(f)\}$.

\paragraph{In-distribution optimal stitched policy.} 
Building on the feasible set $\cF_{\beta}$, we define the {\it in-distribution optimal RTG} $f^{\star}$ as follows,
\begin{align}
     f^{\star}(s, h) := \argmax_{f \in \cF_{\beta}(s, h)} f(s, h),
\end{align}
which represents the in-distribution optimal return-to-go starting from $(s, h)$ in the offline dataset. Based on this, we define the in-distribution optimal stitched policy $\pi_{\beta}^{\star}$ as $ \pi_{\beta}^{\star}(a \mid s, h) := P_{\beta}(a \mid s, h, f^{\star}(s, h))$,
where the policy is conditioned on the optimal return-to-go $f^{\star}(s, h)$. As an extension, we define a broader class of return-conditioned policies $ \pi_f(a \mid s, h) := P_{\beta}(a \mid s, h, f(s, h)),  \forall f \in \cF_{\beta}$.
We denote the subset of conditioning functions that satisfy the consistency condition as $\cF^{\text{Cst}}_{\beta} \subset \cF_{\beta}$. Using this, we present our first theorem, which highlights the superiority of the optimal stitched policy.

\begin{theorem} 
\label{thm:superiority of the optimakl stitched policy}
For any $f \in \cF^{\text{Cst}}_{\beta}$, we have $J(\pi^{\star}_{\beta}) \geq J(\pi_f)$ for all $f \in \cF^{\text{Cst}}_{\beta}$.
In a word, the value achieved by the optimal stitched policy $\pi^{\star}_{\beta}$, equipped with  $f^{\star} \in \cF_{\beta}$, is always at least as good as that of policies constrained by the classical RCSL consistency condition. 
\end{theorem}
\Cref{thm:superiority of the optimakl stitched policy} shows that conditioning function-in-distribution optimal RTG $f^{\star}$ enables stitching with a better trajectory at each time step. However, in practice, the in-distribution optimal RTG $f^{\star}$ is unknown, and consequently, so is the in-distribution optimal stitched policy $\pi^{\star}_{\beta}$.

We propose our algorithm $\algname$ in \Cref{alg-RCSL: deterministic env} 
to construct the in-distribution optimal RTG 
\begin{wrapfigure}{r}{0.54\textwidth}
\centering
\vspace{-15pt}
    \begin{minipage}[t]{\linewidth}
      \begin{algorithm}[H]
\caption{The Reinforced RCSL  ($\algname$)\label{alg-RCSL: deterministic env}}
    \begin{algorithmic}[1]
        \REQUIRE The offline dataset $\cD$.
        \STATE Set $\hat{\pi} = \argmin_{\pi\in\Pi} \hat{L}(\pi)$ following \eqref{def:mle}
        \label{algline: policy estimation}
        \STATE Obtain the in-distribution optimal RTG function estimation, $\hat{f}^{\star}(s,h)$. %
        \label{algline:conditioning function}
        \STATE Receive the initial state $s_1$.
        \FOR{$h=1, \cdots, H$}
            \STATE Establish $\hat{\pi}_{\cD}^{\star}(\cdot|s_h,h)=\hat{\pi}(\cdot|s_h,h,\hat{f}^{\star}(s_h,h))$.
            \STATE Implement $a_h\sim \hat{\pi}_{\cD}^{\star}(\cdot|s_h,h)$ and receive the next state $s_{h+1}$.
        \ENDFOR
    \end{algorithmic}
\end{algorithm}
    \end{minipage}
  \end{wrapfigure}
estimation and estimate the in-distribution optimal stitched policy.  
During training, $\algname$ first follows the RCSL framework by minimizing the empirical negative log-likelihood loss defined in \eqref{def:mle}. 
Then it additionally estimates the in-distribution optimal RTG function (Line \ref{algline:conditioning function}). Here we do not specify a particular $\hat{f}$ estimation procedure, which will be instantiated under specific settings (see examples for the tabular setting in \Cref{subsec:tabular} and for the general function approximation setting in \Cref{subsec:function approximation}). %
During inference, $\algname$ uses the in-distribution optimal RTG estimation as the condition  to 
construct the in-distribution optimal stitched policy estimation $\hat{\pi}_{\cD}^{\star}$.

\section{Finite-Sample Analysis of $\algname$}
\label{sec:case-1 Uniformly Well-Covered Offline Dataset.}

\subsection{Warm-up Analysis for Deterministic Environments}
\label{subsec:tabular}
In this section, we study $\algname$ realization under different environment setups, and provide finite-sample guarantees for variants of $\algname$. 
We start with a {\it deterministic} environment, under which we instantiate and analyze $\algname$  to provide clearer insights into its behavior and advantages. Here, we have a deterministic transition $P$ and deterministic rewards $r_h$, while assuming finite state and action spaces. The initial state distribution $\rho$ and the behavior policy $\beta$ remain stochastic. For notational simplicity, we redefine a trajectory as $\tau = (s_1, a_1, g_1, s_2, a_2, g_2, \dots, s_H, a_H, g_H)$, where $g_h$ represents the RTG at stage $h$.

Given the current state $s_h$ at stage $h$, we define $T_{\cD}(s_h) = \{k \in [N] \mid s_h^{k} = s_h\}$, representing the set of trajectories in the empirical dataset whose state at stage $h$ is $s_h$.
We set the estimation of the in-distribution optimal RTG in \Cref{alg-RCSL: deterministic env} to be $\hat{f}^{\star}(s_h, h) = \argmax_{k \in T_{\cD}(s_h)} g_h^k$. This assigns the \emph{empirical in-distribution optimal RTG} from the dataset to the conditioning function. 
It subsequently determines the \emph{empirical in-distribution optimal stitched policy}: $\hat{\pi}_{\cD}^{\star}(\cdot | s_h, h) = \hat{\pi}(\cdot | s_h, h, \hat{f}^{\star}(s_h, h))$.
Thus, Algorithm \ref{alg-RCSL: deterministic env} effectively follows the steps outlined in Section \ref{sec:basicalg} to learn the in-distribution optimal stitched policy $\pi_\beta^*$ using an empirical dataset $\cD$ instead of the behavior policy itself.

\noindent\textbf{Theoretical guarantee.}
The policy learned by \Cref{alg-RCSL: deterministic env}, denoted as $\hat{\pi}_{\cD}^{\star}$, is an estimate of $\pi^{\star}_{\beta}$. We now analyze its finite-sample theoretical guarantee. At a high level, achieving a reliable estimation requires:
1) the empirical in-distribution optimal  RTG, $\hat{f}^{\star}$, to be accurate, and 
2) sufficient coverage of the offline dataset over the trajectories induced by the in-distribution optimal stitched policy.

We begin by stating a standard assumption on the regularity of the policy class $\Pi$, following \cite{brandfonbrener2022does}.

\begin{assumption}\label{assumption: regular}
For the policy class $\Pi$, we assume it is finite, and 
\begin{itemize}[leftmargin = *]
    \item For all $(a, s, g, h, a', s', g', h')$, $\pi \in \Pi$, we have $|\log \pi(a \mid s, h, g) - \log \pi(a' \mid s', h', g')| \leq c$.
    \item The approximation error of MLE in \eqref{def:mle} is bounded by $\delta_{\text{approx}}$, i.e., $\min_{\pi \in \Pi} L(\pi) \leq \delta_{\text{approx}}$, where $L(\pi) = \mathbb{E}_{s \sim P_{\beta}} \mathbb{E}_{g \sim P_{\beta}(\cdot \mid s)}
        \big[D_{\text{KL}}(P_{\beta}(\cdot \mid s, g) \| \pi(\cdot \mid s, g))\big]$
    is the expected loss.
\end{itemize}
\end{assumption}

Next, we introduce an assumption on the data distribution, which characterizes how well the offline dataset covers the target policy.

\begin{assumption}\label{assumption: data1}
Let $d_{\min}^{\beta} := \min_{h, s} \big\{ d^{\beta}_h(s) \mid d^{\beta}_h(s) > 0 \big\}$
denote the smallest positive entry of the distribution $d^{\beta}$. 
\begin{itemize}[leftmargin = *]
    \item {\bf (Return Coverage)} There exists a constant $\tilde{c} > 0$ such that for all $(s, h) \in \text{dom}(f^{\star})$,  $ P_{\beta}(g_h = f^{\star}(s, h) \mid s_h = s) \geq \tilde{c}$.
    \item {\bf (Distribution Mismatch)} There exists a constant $c_{\beta}^{\star} > 0$ such that for all $(h, s) \in [H] \times \mathcal{S}_h^{\beta}$, we have $d_h^{\star, \beta}(s)/d_h^{\beta}(s) \leq c^{\star}_\beta$,
    where $d^{\star, \beta}_h$ is the occupancy measure on states at step $h$ induced by $\pi_{\beta}^{\star}$.
\end{itemize}
\end{assumption}

\Cref{assumption: data1} imposes a \textit{partial}-type coverage assumption: it only requires the behavior policy (or offline dataset) to cover both the in-distribution optimal return-to-go and the state visitation distribution induced by the optimal stitched policy $\pi^{\star}_{\beta}$. Comparing \Cref{thm:finite sample guarantee for reinforced RCSL - partial coverage} with Corollary~3 of \cite{brandfonbrener2022does}, the term $c^{\star}_{\beta}$ plays a role analogous to $C_f := \sup_{f \in \mathcal{F}_{\beta}^{C}} P_{\pi_f^{\text{RCSL}}}(s)/P_{\beta}(s)$, which captures the worst-case distribution mismatch. Similarly, our term $\tilde{c}$ corresponds to $\alpha_f$, the lower bound on return coverage, ensuring that $P_{\beta}(g = f(s, h) \mid s_h = s) \geq \alpha_f$ for all $f \in \mathcal{F}_{\beta}^{C}$. Hence, \Cref{assumption: data1} is both mild and practically reasonable.

Next, we formally state the theoretical guarantee for  \Cref{alg-RCSL: deterministic env}.
\begin{theorem}
\label{thm:finite sample guarantee for reinforced RCSL - partial coverage}
Under \Cref{assumption: regular,assumption: data1}, if we set $\hat{f}^{\star}(s_h, h) = \argmax_{k \in T_{\cD}(s_h)} g_h^k$ in \Cref{alg-RCSL: deterministic env}, then 
for any $\delta\in(0,1)$, when $ N>\log(SH/\delta)/\log(1-d_{\min}^\beta\cdot \tilde{c})$, with probability at least $1-2\delta$, we have 
    \begin{align*}
    J(\pi_\beta^{\star}) - J(\hat{\pi}_{\cD}^{\star}) \leq O\Big(\frac{c_{\beta}^{\star}H^2}{\tilde{c}} \Big(\sqrt{c}\Big(\frac{\log|\Pi|/\delta}{N} \Big)^{1/4}+ \sqrt{\delta_{\text{approx}}} \Big) \Big).
\end{align*}
\end{theorem}

\Cref{thm:finite sample guarantee for reinforced RCSL - partial coverage} shows that $\algname$ converges to $\pi_{\beta}^{\star}$, which rigorously outperforms the RCSL policies $\pi_f$ for $f \in \mathcal{F}_\beta^{\text{Cst}}$ studied in \cite{brandfonbrener2022does} according to the result in \Cref{thm:superiority of the optimakl stitched policy}. 
It also indicates that the sample complexity of \Cref{alg-RCSL: deterministic env} depends on the approximation error of MLE. This error can be eliminated by selecting a sufficiently expressive function class, such as deep neural networks. Additionally, the convergence rate is $N^{-1/4}$, which is slower than the standard rate of $N^{-1/2}$ commonly established in the offline RL literature. We believe this discrepancy is due to a limitation in the current analysis, and we aim to refine it in future work. 

\subsection{Analysis for Stochastic Environments}
\label{subsec:function approximation}
Next, we consider a more general setting where the state and action spaces are large, and the underlying environment is {\it stochastic}. In this case, we can no longer determine the empirical in-distribution optimal RTG by directly selecting the RTG from offline datasets. To address this challenge, we propose to estimate $\hat{f}^{\star}$ by general function approximation. For now, we do not specify the estimation method for $\hat{f}^{\star}$, it can be instantiated using expectile regression or quantile regression in later sections.
We now outline the assumptions necessary for the theoretical guarantee of \Cref{alg-RCSL: deterministic env} with general function approximation of the in-distribution optimal RTG.

\begin{assumption}\label{assumption: func1}
For the conditioning function $\hat{f}^{\star}$ and the policy class $\Pi$, we assume:
\begin{itemize}[leftmargin = *]%
    \item There exists an error function $\text{Err}(N, \delta, \tilde{c})$ that depends on the sample size $N$ and failure probability $\delta$, such that $\mathbb{E}_{s \sim d_h^{\beta}} \big[(f^{\star}(s,h) - \hat{f}^{\star}(s,h))^2\big] \leq \text{Err}(N,\delta, \tilde{c}), \forall h \in [H]$.
    \item For any $(s, h, \pi) \in \mathcal{S} \times [H] \times \Pi$, given $g_1 \neq g_2$, there exists a constant $\gamma > 0$ such that $\text{TV}(\pi(\cdot \mid s, h, g_1) \| \pi(\cdot \mid s, h, g_2)) \leq \gamma |g_1 - g_2|$,
    where $\text{TV}(\cdot \| \cdot)$ denotes the total variation distance.
\end{itemize}
\end{assumption}

The first condition in \Cref{assumption: func1} ensures that the estimation of the conditioning function is sufficiently accurate. It is not meant to introduce additional constraints, but rather to provide a general and abstract formulation—captured via terms like Err($N,\delta,\tilde{c}$)—that subsumes a wide range of cases.
The second condition guarantees that small errors in the conditioning function do not result in significant divergence in the estimated return-conditioned policy.
Next we state our theorem.

\begin{theorem}
\label{thm:RCSL with general function approximation}
    Assume \Cref{assumption: regular,assumption: data1} hold. Additionally, if the conditioning function $\hat{f}^{\star}$ and the policy class $\Pi$ satisfy \Cref{assumption: func1}, then for any $\delta \in (0,1)$, 
     with probability at least $1-2\delta$, the policy $\hat{\pi}_{\cD}^{\star}$ learned by \Cref{alg-RCSL: deterministic env} satisfies
     \begin{small}
       \begin{align*}
         J(\pi_\beta^{\star}) - J(\hat{\pi}_{\cD}^{\star})&\leq O\Big(\frac{c^{\star}_{\beta}H^2}{\tilde{c}} \Big(\sqrt{c}\Big(\frac{\log|\Pi|/\delta}{N} \Big)^{1/4}+ \sqrt{\delta_{\text{approx}}} \Big)+c_{\beta}^{\star}H^2\gamma\sqrt{\text{Err}(N,\delta, \tilde{c})} \Big).
     \end{align*}       
     \end{small}
\end{theorem}

Compared to the results in \Cref{thm:finite sample guarantee for reinforced RCSL - partial coverage}, \Cref{thm:RCSL with general function approximation} introduces an additional approximation error term, $\text{Err}(N,\delta, \tilde{c})$. While we retain the flexibility to choose our estimation method, we are particularly interested in approaches such as \emph{expectile regression} \citep{wu2024elastic} and \emph{quantile regression} \citep{koenker2001quantile}, which can potentially achieve error bounds of order $O(1/N)$. We further analyze their properties in the next section.

\section{Practical Implementation: Expectile v.s. Quantile Regression}
\label{sec:practical}
In this section, we study how different function approximation procedure, especially those utilized in literature of the conditioning function $\hat f^*$, would affect the learned algorithm, under the simple {\it tabular} setting. Existing literature \citep{wu2024elastic, zhuang2024reinformer} leverage the expectile regression \citep{newey1987asymmetric} due to its simplicity. It returns the empirical conditioning function by $\hat{f}^{\star}=\argmin_{\tilde{f}\in\cF} \sum_{k=1}^K \big[L_2^{\alpha}(g_h^k - \tilde{f}\big(s_h^k,h\big)
    \big)\big]$,
where $L_2^{\alpha}(u) = |\alpha- \ind(u<0)|u^2$ and $\alpha$ is the hyperparameter in order to control how close expectile regression is to the vanilla $L_2$ regression. 
However, it is trivial to show that the expectile estimator with $\alpha \neq 1$ could lead to out of distribution RTG. Formally, there exists a tabular MDP and a behavior policy $\beta$ such that the $\algname$ with the expectile regression for $\hat{f}^{\star}$ estimation can not find the optimal policy $\pi_\beta^{\star}$ for sure.
We postpone the proof to 
\Cref{sec: hard instances for expectile regression}. The take away message is that the $L_2$ loss of the expectile regression leads to out-of-distribution returns when $\alpha\neq 1$, due to the fact that $L_2$ loss is less robust to the noise.

To address this issue, we consider the quantile regression \citep{koenker2001quantile}, which returns the empirical conditioning function by $\hat{f}^{\star}=\argmin_{\tilde{f}\in\cF} \sum_{k=1}^K L_1^{\alpha}(g_h^k - \tilde{f}\big(s_h^k,h\big)\big)$, where $L_1^{\alpha}(u) = |\alpha- \ind(u<0)|\cdot|u|$ is the $L_1$ loss. 
Generally speaking, the $L_1$ loss is more robust to the noise of the return, which makes \Cref{alg-RCSL: deterministic env} with quantile regression for $\hat{f}^{\star}$ estimation better than its expectile regression counterpart. In the following theorem, we state that with a large sample size, the $\algname$ with quantile regression can exactly recover the in-distribution optimal RTG, finding the optimal policy. In contrast, $\algname$ with expectile regression introduces bias in $f^{\star}$ estimation, learns out-of-distribution RTGs, and fails to find the optimal policy. We make the following assumption here. 
\begin{assumption}\label{assumption: environment}
We assume that the environment is deterministic. Besides, we assume there is no tie in the RTG: for any $h\in[H]$, $\forall \tau^1,\tau^2\in T_{\beta}$, such that $s^1_h=s^2_h$ and $a^1_h\neq a^2_h$, we have $g^1_h\neq g^2_h$. 
\end{assumption}

\Cref{assumption: environment} essentially suggests a setting where all trajectories generated by the behavior policy $\beta$ can be `ranked' based on their RTG. Under \Cref{assumption: environment}, the optimal policy $\pi^{\star}_\beta$ can be represented by a trajectory starting from the initial state, which enables a simplified analysis of how various training methods influence $\hat f^{\star}$. 

\begin{theorem}
\label{thm:theoretical guarantee for reinforced RCSL with quantile regression}
With \Cref{assumption: environment}, we further assume that for all $(s,h)\in \text{dom}(f^{\star})$, there exists a constant $\tilde{c}$, such that
    $P_{\beta}(g_h=f^{\star}(s,h)|s_h=s)\geq \tilde{c}$. We set $\alpha > 1-\tilde{c}/2$ in the $L_2^{\alpha}$ loss with quantile regression.
Then for any $\delta\in(0,1)$, when $N\geq \max\{\frac{2}{d_{\min}^{\beta,2}}\log\frac{2SH}{\delta}, \frac{4}{\tilde{c}^2d_{\min}^\beta}\log\frac{2SH}{\delta}\}$
with probability at least $1-\delta$, we have $J(\pi^{\star}_{\beta}) = J(\hat{\pi}_{\cD}^{\star})$, where $\hat{\pi}_{\cD}^{\star}$ is the policy learned by \Cref{alg-RCSL: deterministic env}.
\end{theorem}
\Cref{thm:theoretical guarantee for reinforced RCSL with quantile regression} states that 
with sufficient data and a properly chosen hyperparameter $\alpha$, \Cref{alg-RCSL: deterministic env} with quantile regression instantiation can find the in-distribution optimal stitched policy, addressing the issue of expectile regression mentioned above. Notably, the choice of $\alpha$ depends on the data coverage level $\tilde{c}$. In the next section, we will validate this observation using experimental study.

\section{Experiments}
In this section, we conduct several numerical experiments to answer the following two questions: \textit{How does different choice of hyperparameter $\alpha$ affect $\algname$'s stitching ability?} and \textit{How does $\algname$ compare to existing extensions of RCSL in literature?}

\subsection{An Illustration of $\algname$'s Stitching Ability }

To answer the first question, we conduct a simulation study in PointMaze to showcase the stitching ability of \Cref{alg-RCSL: deterministic env} with expectile regression and quantile regression, and study how different choices of $\alpha$ affect stitching. The simulated environment  is a point-mass navigation task. Red dots represent   
\begin{wrapfigure}[16]{r}{0.35\textwidth}
    \centering
    \includegraphics[width=\linewidth]{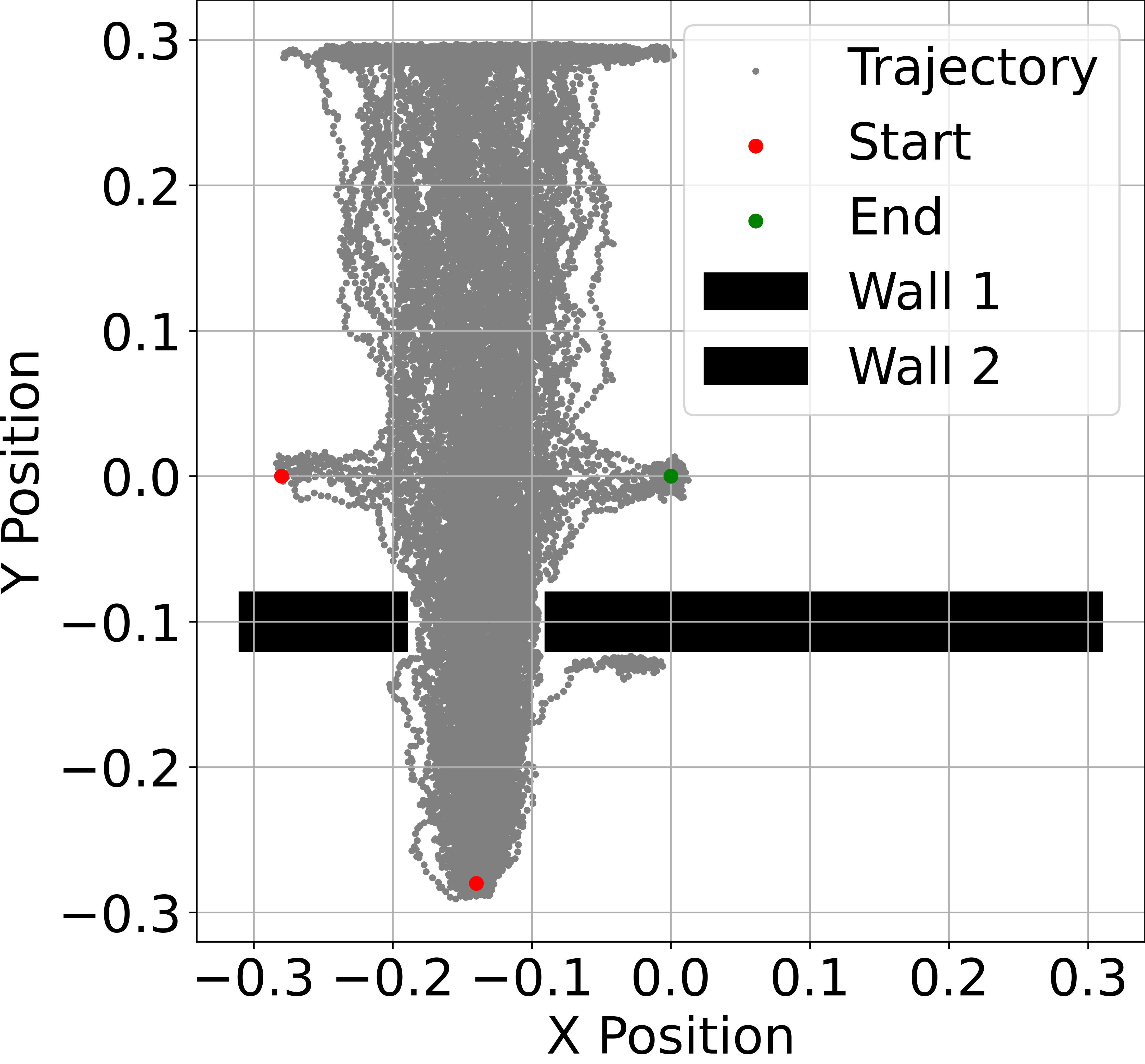}
    \caption{The simulated PointMaze environment}
    \label{fig: simulation env}
\end{wrapfigure}
starting points and the green dot is the goal state. The agent learns a policy to reach the goal from various starting positions. The offline dataset contains two types of trajectories:
{\bf Type I}: starting from the left red point and going directly to the goal;
{\bf Type II}: starting from the bottom red point and moving upward without reaching the goal. To succeed when starting from the bottom red point, the agent must {\bf learn to stitch}—i.e., combine information from Type II and Type I trajectories to reach the goal. We also inject action noise at each step to evaluate generalization.

We vary the proportion of Type I trajectories to simulate different levels of coverage of the optimal return-to-go in the offline dataset, which corresponds to $\tilde{c}$ in \Cref{assumption: data1}. To analyze how this affects performance, we test \Cref{alg-RCSL: deterministic env} leveraging expectile regression and quantile regression with various values of the hyperparameter $\alpha$ for $\hat{f}^{\star}$ estimation. We find that with {\bf 10\%} Type I trajectories, $\alpha = 0.95$ enables successful stitching (\Cref{fig:our_expectile_0.95_point_mass-stitch-easy_0.1,fig:our_quantile_0.95_point_mass-stitch-easy_0.1}), while $\alpha = 0.85$ fails (\Cref{fig:our_expectile_0.85_point_mass-stitch-easy_0.1,fig:our_quantile_0.85_point_mass-stitch-easy_0.1}); With {\bf 1\%} Type I trajectories, a higher $\alpha = 0.99$ is required to achieve stitching (\Cref{fig:our_expectile_0.99_point_mass-stitch-easy_0.01,fig:our_quantile_0.99_point_mass-stitch-easy_0.01}), whereas $\alpha = 0.90$ fails (\Cref{fig:our_expectile_0.9_point_mass-stitch-easy_0.01,fig:our_quantile_0.9_point_mass-stitch-easy_0.01}). Thus, the results are consistent with our theoretical findings in \Cref{thm:theoretical guarantee for reinforced RCSL with quantile regression}, 
the hyperparameter $\alpha$ should be chosen according to the underlying coverage factor $\tilde{c}$. 

\begin{figure}[h]
    \centering
    \subfigure[Expectile ($\alpha=0.85$)]{\includegraphics[width=0.24\linewidth]{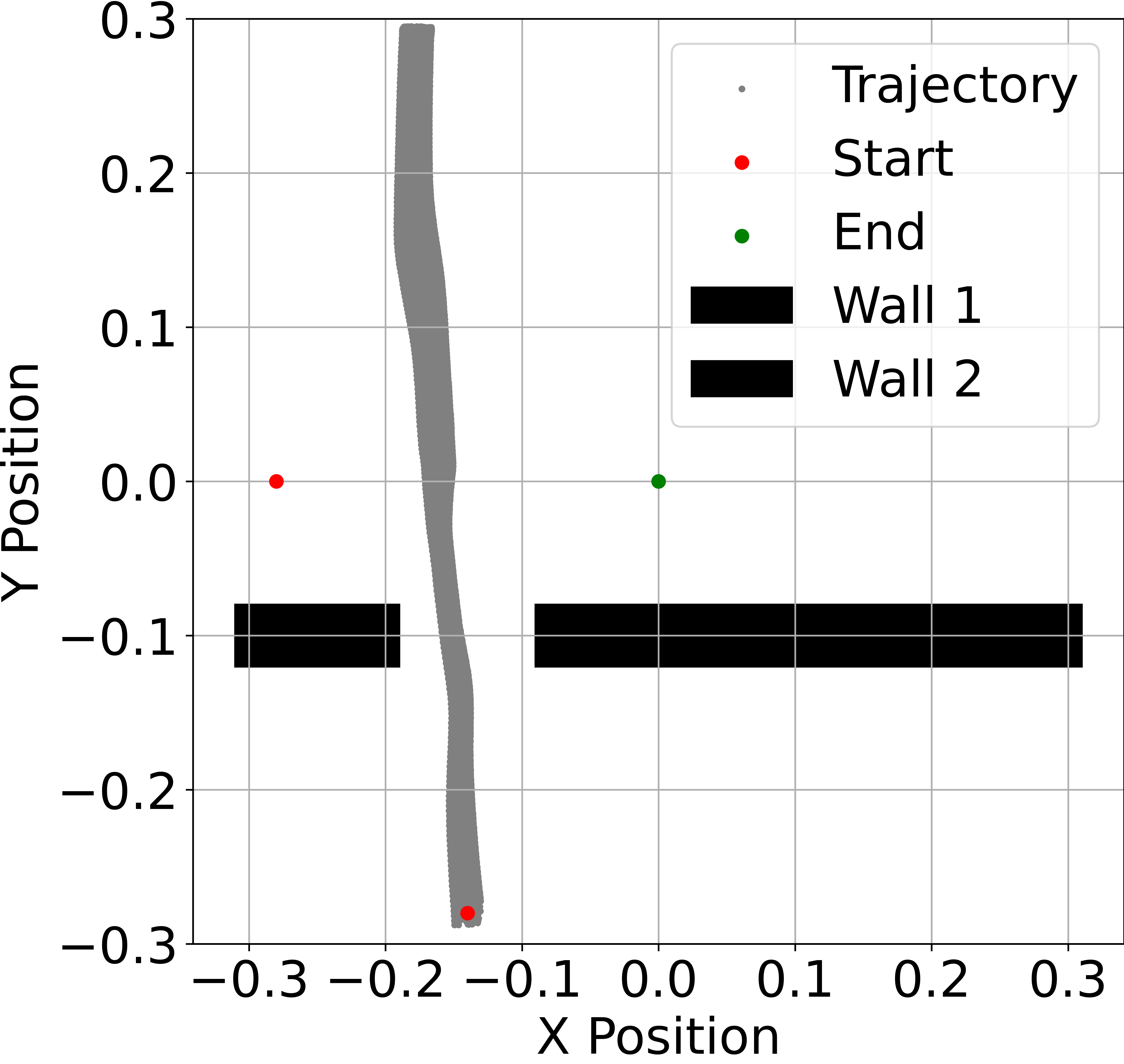}\label{fig:our_expectile_0.85_point_mass-stitch-easy_0.1}}
    \subfigure[Expectile ($\alpha=0.95$)]{\includegraphics[width=0.24\linewidth]{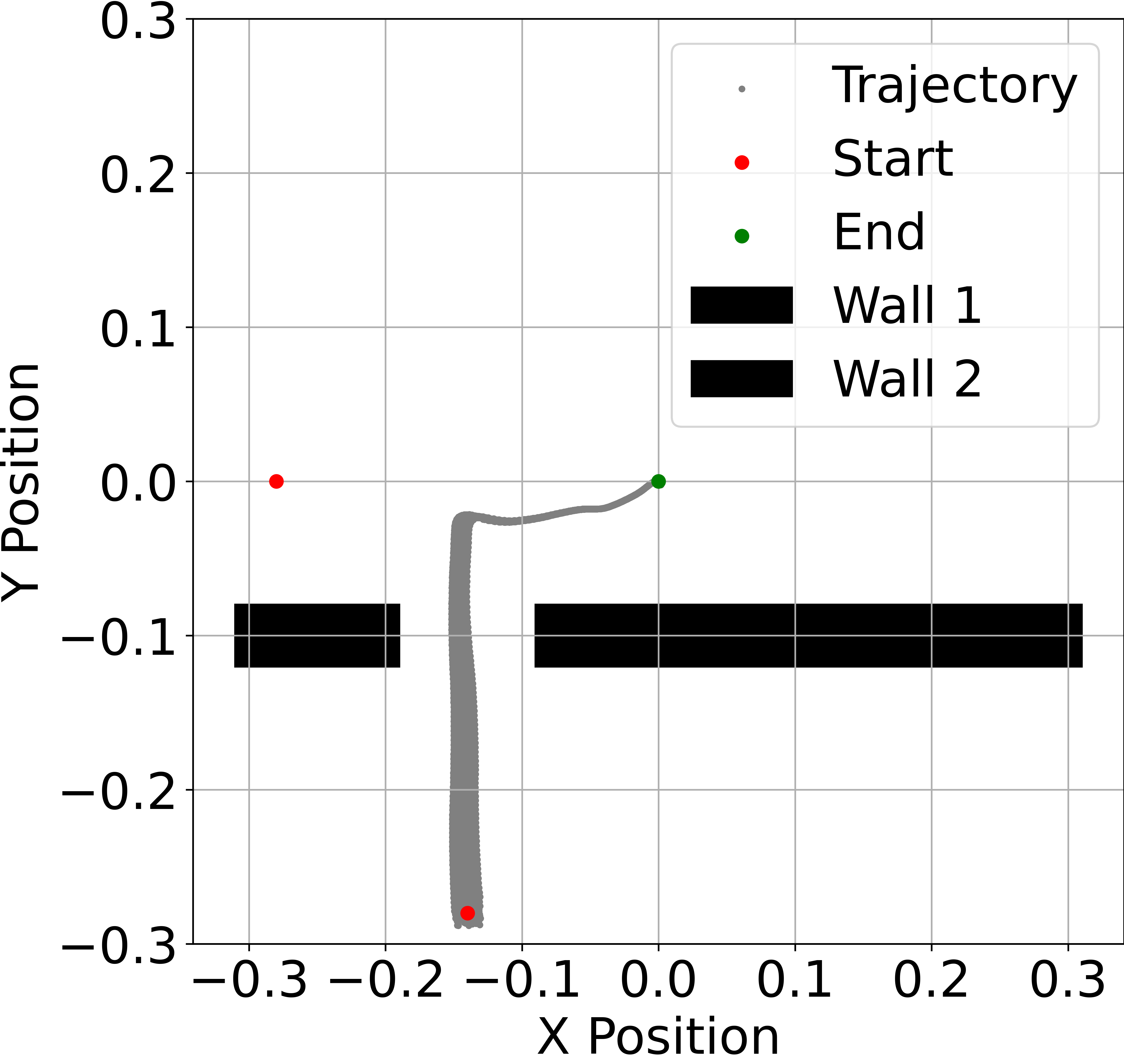}\label{fig:our_expectile_0.95_point_mass-stitch-easy_0.1}}
    \subfigure[Quantile ($\alpha=0.85$)]{\includegraphics[width=0.24\linewidth]{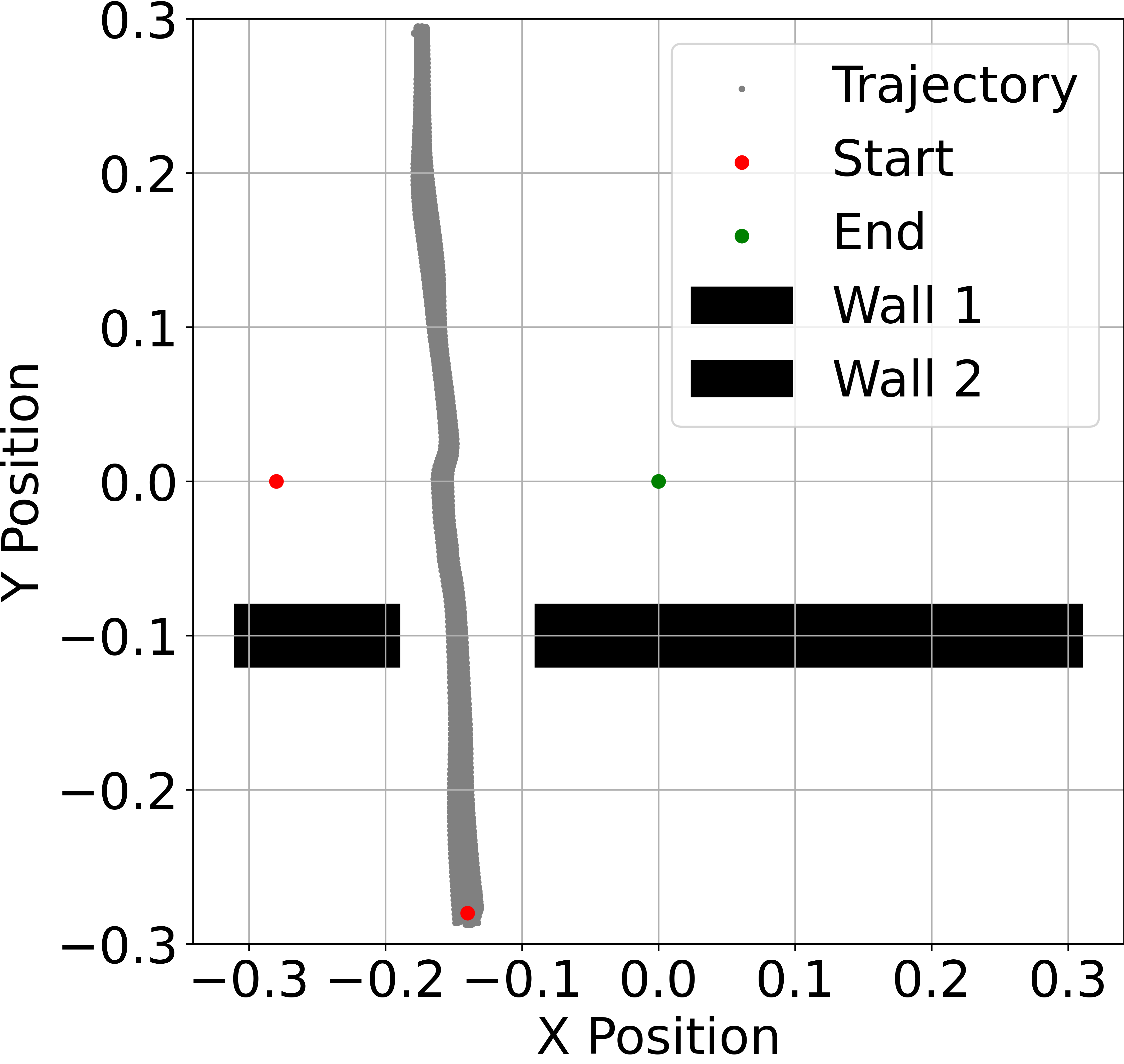}\label{fig:our_quantile_0.85_point_mass-stitch-easy_0.1}}
    \subfigure[Quantile ($\alpha=0.95$)]{\includegraphics[width=0.24\linewidth]{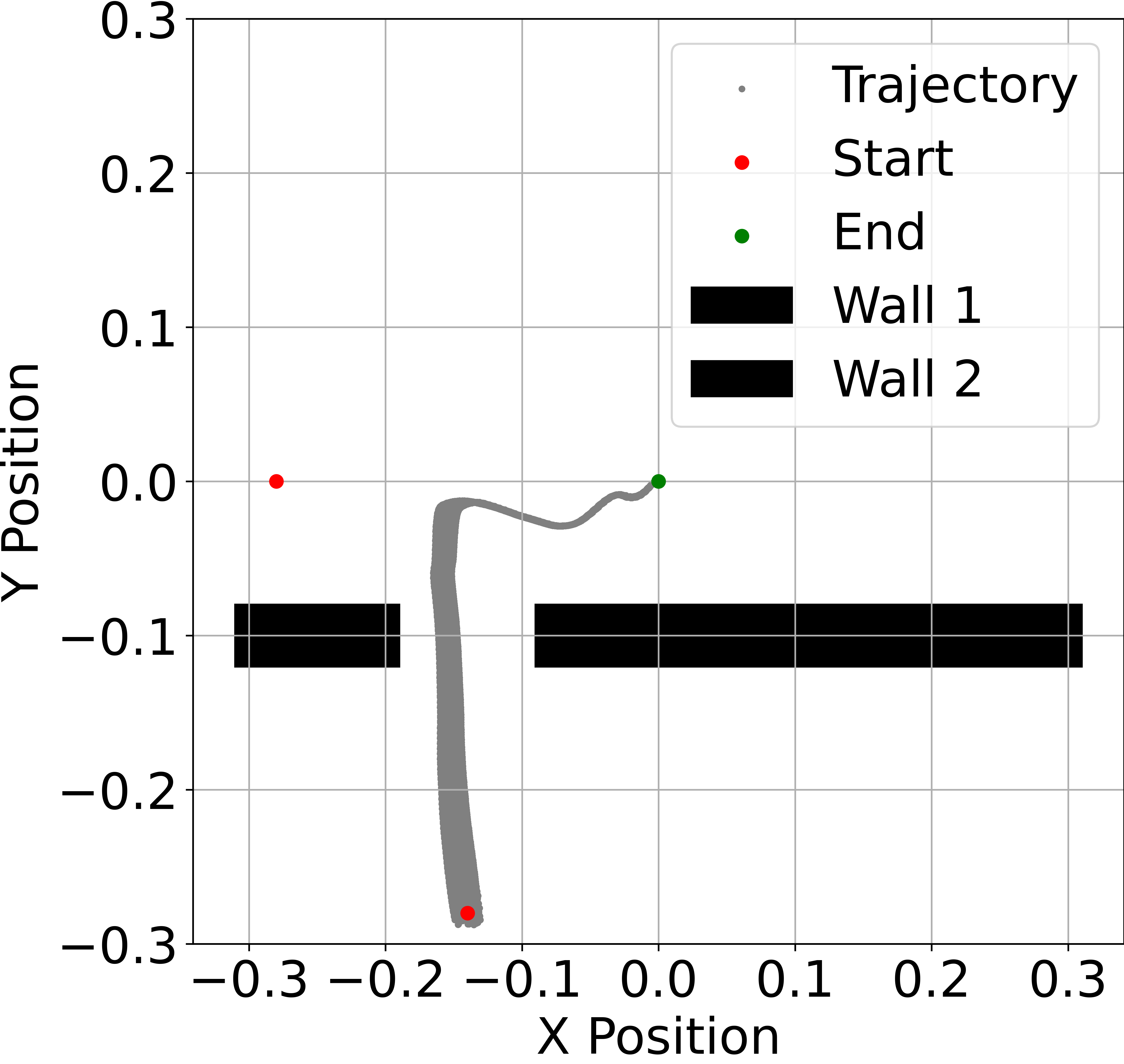}\label{fig:our_quantile_0.95_point_mass-stitch-easy_0.1}}
    \caption{Illustration of the stitching ability of $\algname$ - Proportion of type I trajectories in the dataset is set to 0.1. $\algname$ with $\alpha=0.85$ fails to stitch, while with $\alpha=0.95$ succeeds  to stitch.}
    \label{fig:enter-label}
\end{figure}
\begin{figure}
    \centering
    \subfigure[Expectile ($\alpha=0.9$)]{\includegraphics[width=0.24\linewidth]{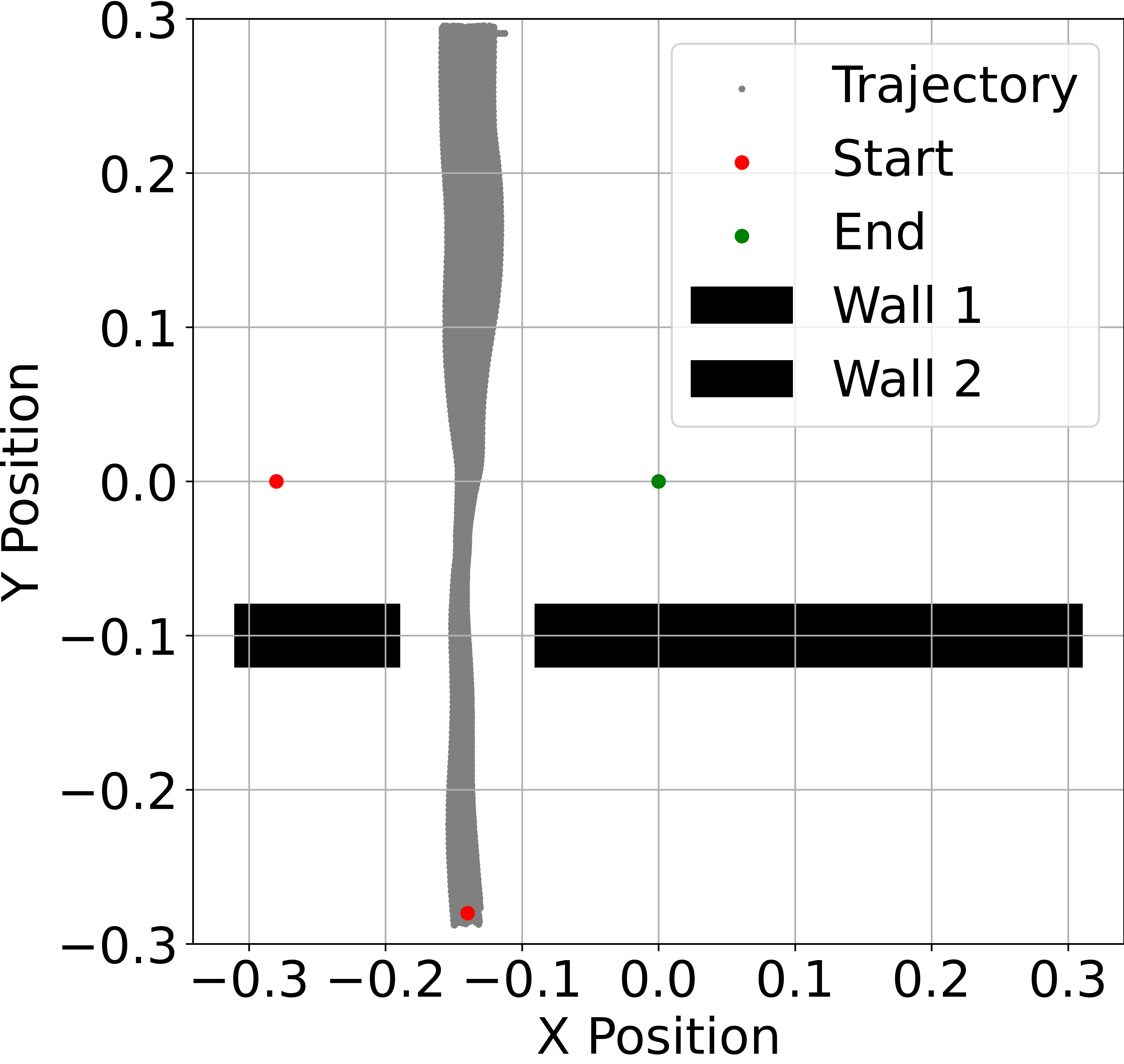}\label{fig:our_expectile_0.9_point_mass-stitch-easy_0.01}}
    \subfigure[Expectile ($\alpha=0.99$)]{\includegraphics[width=0.24\linewidth]{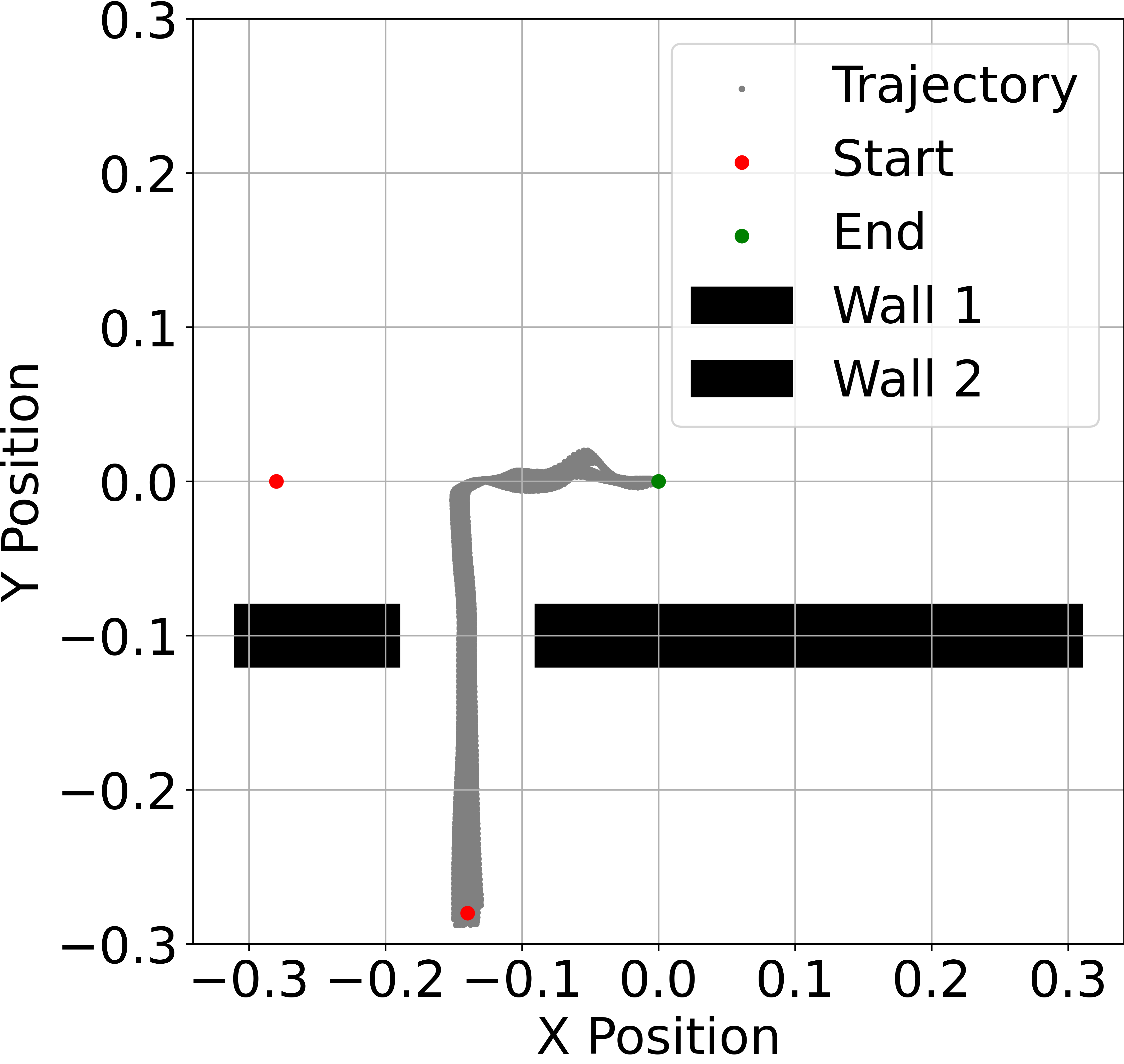}\label{fig:our_expectile_0.99_point_mass-stitch-easy_0.01}}
    \subfigure[Quantile ($\alpha=0.9$)]{\includegraphics[width=0.24\linewidth]{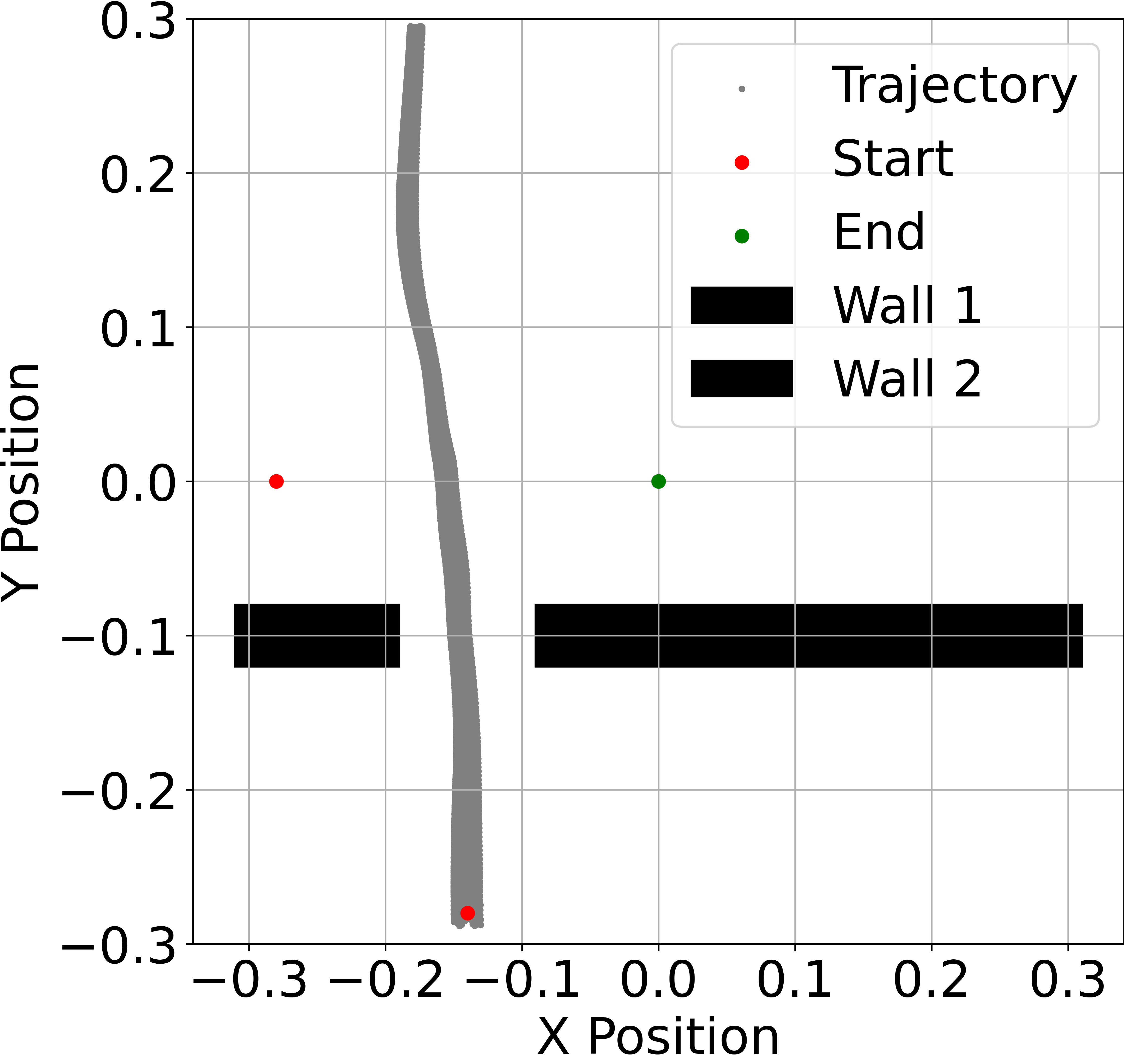}\label{fig:our_quantile_0.9_point_mass-stitch-easy_0.01}}
    \subfigure[Quantile ($\alpha=0.99$)]{\includegraphics[width=0.24\linewidth]{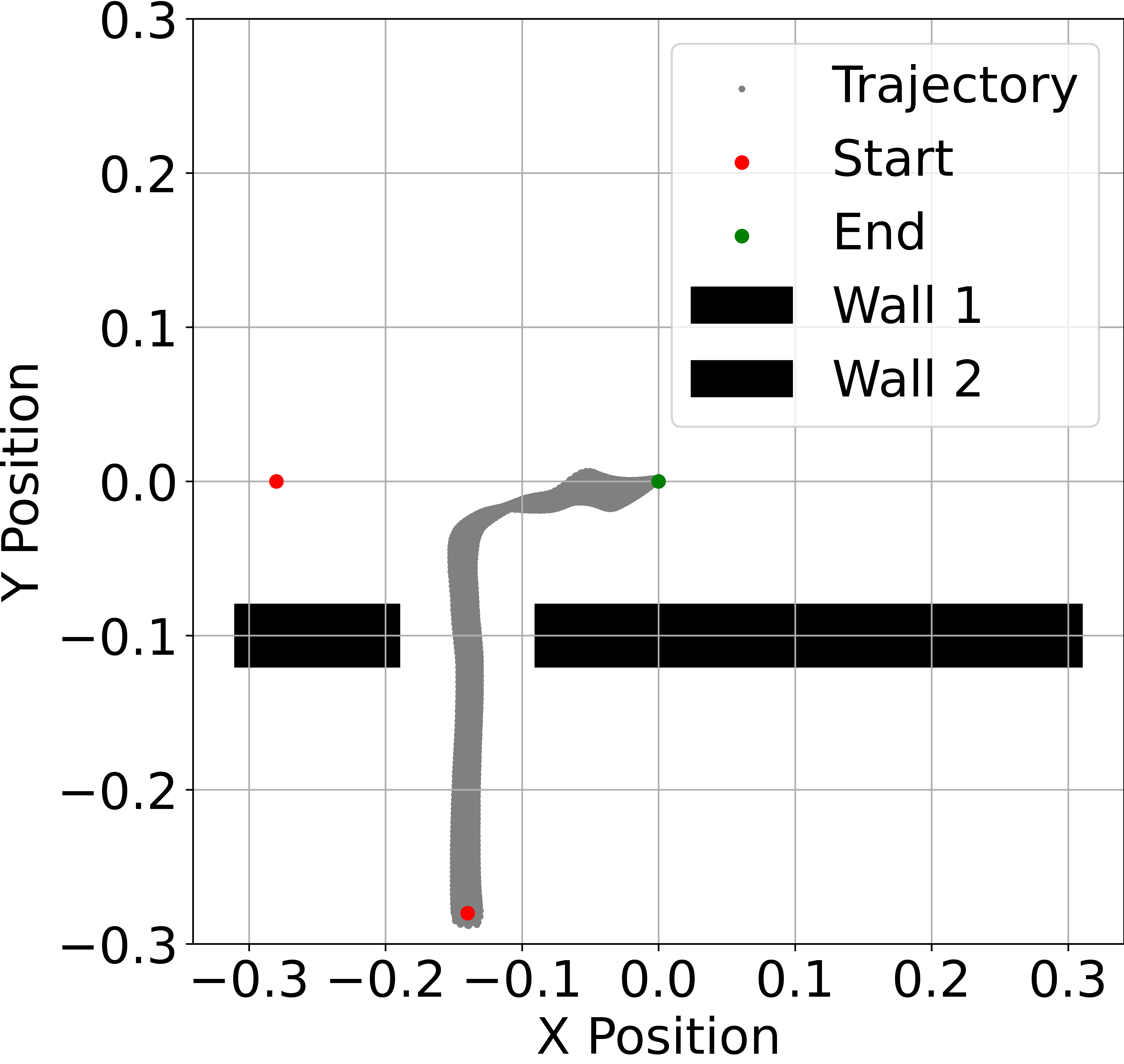}\label{fig:our_quantile_0.99_point_mass-stitch-easy_0.01}}
    \caption{Illustration of Stitching - Proportion of type I trajectories in the offline dataset is set to 0.01. $\algname$ with $\alpha=0.85$ fails to stitch, while with $\alpha=0.95$ succeeds  to stitch. 
    }
\end{figure}

\subsection{D4RL Benchmark}
To answer the second question, we test several $\algname$ variants under the D4RL Gym (halfcheetah, hopper and walker2d) and Antmaze environments \citep{fu2020d4rl}. 
Details on experiment setup, implementation are postponed to \Cref{sec: Experiments Details}.
\subsubsection{D4RL Gym}

{\bf $\algname$ with RvS.} We implement \Cref{alg-RCSL: deterministic env} with expectile regression and quantile regression based on the RvS \citep{emmons2022rvs}, heuristically motivated from the algorithm design and theoretical results developed, that can deal with large state and action spaces leveraging powerful function approximations. 
We compare our $\algname$ with the vanilla RvS as a baseline method. For RvS, achieving optimal performance during inference requires searching for the target RTGs \citep{emmons2022rvs}. However, this process is often impractical in real-world scenarios.
To address this, we select three appropriate target RTG {\it fraction ratios} as $0.7,0.9,1.1$ to guide the RvS instead, which means the initial target RTG will be 
\begin{align}
\label{eq: target RTG}
  \text{RTG}_{\text{inital}} = (\text{RTG}_{\max}-\text{RTG}_{\min}) * \text{fraction} + \text{RTG}_{\min},
\end{align}  
where $\text{RTG}_{\max}$ and $\text{RTG}_{\min}$ are the maximum RTG and minimum RTG from the random policy and expert policy in corresponding environments respectively.
Experiment results are shown in \Cref{tab:reusls_fixed_goal_0.99}. 
We can conclude that $\algname$ outperforms the RvS framework across all fixed target RTG fraction ratios. This demonstrates the effectiveness and robustness of our approach.
\begin{table*}[t!]
\center
\caption{Normalized score on D4RL Gym for RvS and $\algname$ with expectile ($\alpha=0.99$) and quantile ($\alpha=0.99$) regression respectively. The inference of RvS is conditioned on three target RTGs defined in \eqref{eq: target RTG} based on three {\it fraction ratios} (0.7, 0.9, and 1.1 respectively). We report the mean and standard deviation of the normalized score for five seeds. 
}
\resizebox{1\textwidth}{!}{
\label{tab:reusls_fixed_goal_0.99}
\begin{tabular}{cccccc}
\toprule
\multirow{2}{*}{\textbf{Dataset}} & \multicolumn{3}{c}{\textbf{RVS}}                                            & \multirow{2}{*}{\textbf{$\algname$-Expectile}} & \multirow{2}{*}{\textbf{$\algname$-Quantile}} \\ \cmidrule(lr){2-4}
                                  & \multicolumn{1}{c}{0.7} & \multicolumn{1}{c}{0.9} & \multicolumn{1}{c}{1.1} &                                 &                                \\ \midrule
halfcheetah-medium                & 43.10$\pm$0.64        &36.78$\pm$3.04
                   & 25.47$\pm$4.67
                  & 42.09$\pm$0.50
                        & 42.24$\pm$0.35
                     \\
halfcheetah-medium-replay         & 15.24$\pm$7.60
                  & 9.59$\pm$6.12
            & 5.17$\pm$4.72
          & 38.07$\pm$0.91
                      & 38.71$\pm$0.46 
                    \\
halfcheetah-medium-expert         & 85.48$\pm$1.32
                & 90.97$\pm$1.26 
              & 91.09$\pm$1.11
        & 90.95$\pm$1.18
                     & 92.25$\pm$0.62
                     \\ \midrule
hopper-medium                     & 52.05$\pm$3.90	
                 & 47.91$\pm$4.72	
              & 38.49$\pm$13.57	
           & 54.49$\pm$4.18	
                        & 53.00$\pm$9.62
                     \\
hopper-medium-replay           & 52.53$\pm$22.07	
                 & 53.49$\pm$29.42	
             & 34.30$\pm$14.06 
           & 46.83$\pm$14.55	
                       & 53.27$\pm$20.74
                   \\
hopper-medium-expert             & 65.00$\pm$6.82	
             & 96.18$\pm$20.65	
            & 106.54$\pm$8.25	
            & 106.62$\pm$8.17	
                      & 100.75$\pm$10.92 
                     \\ \midrule
walker2d-medium              & 70.94$\pm$4.11	
               & 72.25$\pm$2.88	
             & 68.82$\pm$3.79	
           & 71.18$\pm$4.22	
                        & 71.66$\pm$4.33 
                      \\
walker2d-medium-replay            & 41.24$\pm$12.86	
              & 47.25$\pm$17.79	
               & 26.58$\pm$20.09	
              & 36.34$\pm$10.52	
                      & 44.85$\pm$9.00
                 \\
walker2d-medium-expert          &  65.06$\pm$1.15	
              & 69.23$\pm$3.70	
                 & 106.22$\pm$0.58	
               & 101.03$\pm$6.29	
                    & 105.18$\pm$0.70 
                   \\ \midrule
Total           & 490.62		

              & 523.66		

                 &502.68		

               & 587.61		
                    & {\bf 601.91} 
                   \\ \bottomrule
\end{tabular}}
\end{table*}

\noindent\textbf{Ablation study}
The hyperparameter $\alpha$ in both expectile and quantile regression controls how these methods emphasize different regions of the return distribution. Our experiments systematically varied $\alpha$ across ${0.9, 0.99, 0.999}$ to study its impact on policy performances. As shown in \Cref{tab:reusls_fixed_goal_diffterent-alpha}, the performance of our methods initially increases as $\alpha$ increases, but with $\alpha = 0.999$, the performance of $\algname$-Quantile does not improve further. Predicting higher RTGs generally leads to better performance when the conditioning function is learned within a reasonable range. However, extreme outliers such as those introduced with $\alpha = 0.999$ in $\algname$-Expectile can hurt the performance.
\begin{table*}[t!]
\caption{{Normalized score on D4RL Gym for our methods with different $\alpha$ in expectile and quantile. We report the mean and standard deviation of the normalized score for five seeds.}}
\center
\resizebox{1\textwidth}{!}{
\label{tab:reusls_fixed_goal_diffterent-alpha}
\begin{tabular}{ccccccc}
\toprule
\multirow{2}{*}{\textbf{Dataset}}  & \multicolumn{3}{c}{\textbf{$\algname$-Expectile}} & \multicolumn{3}{c}{\textbf{$\algname$-Quantile}}                \\
\cmidrule(lr){2-4}\cmidrule(lr){5-7}
                          & 0.9      & 0.99     & 0.999   & 0.9    & 0.99   & 0.999 \\ \midrule
halfcheetah-medium        & 42.04$\pm$0.36

         &  42.08$\pm$0.49

        &    42.42$\pm$0.39

     & 42.28 $\pm$ 0.26

       & 42.24$\pm$0.35

     & 42.47$\pm$0.42

                      \\
halfcheetah-medium-replay & 35.89$\pm$0.29

         &  38.07$\pm$0.91

        &    38.13$\pm$1.81

     & 37.93$\pm$0.30

       & 38.71$\pm$0.46

       & 38.77$\pm$0.52     

                  \\
halfcheetah-medium-expert & 89.85$\pm$1.58

         &   90.95$\pm$1.18

       &   90.90$\pm$1.09

      &  91.70$\pm$1.09
      &   92.25$\pm$0.62

     &  91.81$\pm$0.98

                     \\ \midrule
hopper-medium             & 57.01$\pm$1.73

         &54.49$\pm$4.18

          &  46.44$\pm$4.42

       &  57.24$\pm$2.63

     &  53.00$\pm$9.62

      &  55.62$\pm$8.49

                      \\
hopper-medium-replay      & 32.26$\pm$10.57

         &   46.83$\pm$14.55

       &  53.90$\pm$13.73

       &  34.08$\pm$4.96

     & 53.27$\pm$20.74
     
       & 72.60$\pm$14.08

                        \\
hopper-medium-expert      & 102.58$\pm$4.93

         &  106.62$\pm$8.17

        &    98.57$\pm$16.39

     &  104.04$\pm$7.67

      &  100.75$\pm$10.92

       &  92.24$\pm$30.77  

                    \\ \midrule
walker2d-medium           & 71.29$\pm$3.61

         &   71.18$\pm$4.22

       &    71.13$\pm$2.31

     & 70.23$\pm$5.04

      &  71.66$\pm$4.33

      &  71.94$\pm$4.75

                    \\
walker2d-medium-replay    & 29.87$\pm$7.25

         & 36.34$\pm$10.52

         &   39.46$\pm$19.24

      &  16.60$\pm$5.60

      & 44.85$\pm$9.00

       & 55.62$\pm$5.00

                      \\
walker2d-medium-expert    & 60.26$\pm$15.55

         &  101.04$\pm$6.29

        &    105.08$\pm$0.93

     &  104.73$\pm$1.79

      & 105.18$\pm$0.70

        &  102.21$\pm$5.14

                      \\ \midrule
Total                     & 521.05
   &  587.61
  &  586.03 & 558.84
  & 601.91
  & {\bf 623.29}
  
                       \\ \bottomrule
\end{tabular}}
\end{table*}

{\bf $\algname$ with DT} We extend our $\algname$ framework to DT based methods. In particular, we use the transformer architecture of DT to learn the policy $\pi$ in Line \ref{algline: policy estimation} of \Cref{alg-RCSL: deterministic env}. We use MLP to conduct quantile regression and expectile regression to learn the conditioning function in Line \ref{algline:conditioning function} of \Cref{alg-RCSL: deterministic env}. We call this method DT-R2CSL.
Experiment results are shown in  \Cref{tab:DT}.
We can conclude that in most cases the best performance belongs to our proposed DT-extensions.

\begin{table}[t!]
\caption{{Normalized score on D4RL Gym for DT, DT-$\algname$ with expectile ($\alpha$ = 0.99) and quantile ($\alpha$ = 0.99) regression respectively. We report the mean and standard deviation of the normalized score for five seeds.
}}
\center

\label{tab:DT}
\begin{tabular}{cccc}
\toprule
\textbf{Dataset}                   & \textbf{DT-$\algname$-Expectile} & \textbf{DT-$\algname$-Quantile} & \textbf{DT}            \\ \midrule
halfcheetah-medium       &          43.23$\pm$0.26                &   43.21$\pm$0.09                   & 42.6$\pm$0.1     \\
halfcheetah-medium-replay &           38.17$\pm$1.10             &      37.17$\pm$1.68	                 &  36.6$\pm$0.8    \\
halfcheetah-medium-expert &         88.03$\pm$2.14                &    88.56$\pm$1.99	                    & 86.8$\pm$1.3     \\ \midrule
hopper-medium             &        68.95$\pm$11.70          &   	        70.24$\pm$11.80     &  67.6$\pm$1.0   \\
hopper-medium-replay      &        83.21$\pm$4.53                 &      82.86$\pm$5.93                  &  82.7$\pm$7.0    \\
hopper-medium-expert      &        104.13$\pm$3.45          &        105.50$\pm$2.53       & 107.6$\pm$1.8    \\ \midrule
walker2d-medium           &           82.88$\pm$1.70               &       81.50$\pm$1.37                 &  74$\pm$1.4     \\
walker2d-medium-replay    &           70.03$\pm$3.23              &       69.69$\pm$4.15                 &  66.6$\pm$3.0    \\
walker2d-medium-expert    &        109.59$\pm$0.66                 &  109.09$\pm$0.83                      &  108.1$\pm$0.2   \\ \midrule
Total                   &        {\bf 688.22}                 &          687.83             &  672.6          \\ \bottomrule
\end{tabular}
\end{table}

\begin{table}[t!]
\caption{{Normalized score on D4RL Gym for QT, DP-$\algname$ with expectile ($\alpha$ = 0.99) and quantile ($\alpha$ = 0.99) regression respectively. We report the mean and standard deviation of the normalized score for five seeds.
}}
\center
\label{tab: QT}
\begin{tabular}{cccc}
\toprule
\textbf{Dataset}                & \textbf{DP-$\algname$-Expectile} & \textbf{DP-$\algname$-Quantile} & \textbf{QT}            \\ \midrule
halfcheetah-medium          & 50.71$\pm$0.11          & 51.11$\pm$0.32         & 51.4$\pm$0.4  \\
halfcheetah-medium-replay   & 48.40$\pm$0.48          & 48.42$\pm$0.30         & 48.9$\pm$0.3  \\
halfcheetah-medium-expert   & 82.86$\pm$4.99          & 83.78$\pm$7.11         & 96.1$\pm$0.2  \\ \midrule
hopper-medium               & 74.58$\pm$13.68         & 71.01$\pm$5.56         & 96.9$\pm$3.1  \\
hopper-medium-replay        & 98.92$\pm$0.43          & 98.45$\pm$0.93         & 102$\pm$0.2   \\
hopper-medium-expert       & 112.31$\pm$0.61         & 112.73$\pm$0.28        & 113.4$\pm$0.4 \\ \midrule
walker2d-medium              & 87.82$\pm$0.27          & 89.50$\pm$4.94         & 88.8$\pm$0.5  \\
walker2d-medium-replay      & 97.37$\pm$2.41          & 98.11$\pm$1.26         & 98.5$\pm$1.1  \\
walker2d-medium-expert      & 110.09$\pm$0.51         & 111.45$\pm$1.21        & 112.6$\pm$0.6 \\ \midrule
Total                   & 763.11                 & 764.55                  & {\bf 808.2}          \\ \bottomrule
\end{tabular}
\end{table}

{\bf $\algname$ incorporating dynamic programming} We extend our framework to hybrid methods that incorporate dynamic programming components. In particular, we implement a new method, DP-R2CSL, which integrates the policy learning module from QT \citep{hu2024q}. QT remains, to the best of our knowledge, the state-of-the-art on D4RL benchmarks. This module leverages a pre-learned Q-value to regularize the cross-entropy loss used in policy estimation. 
Experiment results are shown in  \Cref{tab: QT}.
We observe that DP-R$^2$CSL significantly outperforms DT-R$^2$CSL, and performs comparably to QT across all settings except hopper-medium and halfcheetah-medium-expert. This is expected, as QT shares a key feature with our RCSL framework—namely, the use of an optimal conditioning function. Specifically, QT selects the return-to-go (RTG) that maximizes the Q-value as its conditioning input (see Section 3.3 of \cite{hu2024q} for details), which aligns with our principle of selecting the in-distribution optimal RTG. 

\begin{table*}[t!]
\center
\caption{{Normalized score on D4RL AntMaze for RvS and $\algname$ with expectile ($\alpha=0.99$) and quantile ($\alpha=0.99$) regression respectively. The inference of RvS is conditioned on three target RTGs. We report the mean and standard deviation of the normalized score for five seeds. }}\resizebox{1\textwidth}{!}
{
\label{tab:reusls_fixed_goal_0.99_additional}
\begin{tabular}{cccccc}
\toprule
\multirow{2}{*}{\textbf{Dataset}} & \multicolumn{3}{c}{\textbf{RVS}}                                            & \multirow{2}{*}{\textbf{$\algname$-Expectile}} & \multirow{2}{*}{\textbf{$\algname$-Quantile}} \\ \cmidrule(lr){2-4}
                                  & \multicolumn{1}{c}{0.7} & \multicolumn{1}{c}{0.9} & \multicolumn{1}{c}{1.1} &                                 &                                \\ \midrule
umaze               &  54.5$\pm$5.83
      & 57.2$\pm$8.64

                   &  60.2$\pm$10.26

                  &  61.4$\pm$5.97

                        &  64.8$\pm$3.82

                     \\
umaze-diverse         &  
54.6$\pm$6.71
               &  56.5$\pm$4.65
            &  53.7$\pm$3.95

          & 57.9$\pm$2.99

                      &  57.8$\pm$3.19

                    \\
medium-play      &  
 0.1$\pm$0.27
               &  0.1$\pm$0.22

              & 0.1$\pm$0.22

        &  0.3$\pm$0.27

                     &  0.4$\pm$0.42

                     \\ 
medium-diverse                        &  0.2$\pm$0.27 
                 &  0.2$\pm$0.27
              &  0.2$\pm$0.27
           &   0.2$\pm$0.27
                        &  0.5$\pm$0.35
                     \\ \midrule

Total             &  109.4	
               &  	114.0
             &  114.2
           &  119.8
                        &  {\bf 123.5}
\\
\bottomrule
\end{tabular}}
\end{table*}

\subsubsection{D4RL AntMaze}
In this section, we present the experiment results on AntMaze. 

\paragraph{$\algname$ with RvS.} We implement \Cref{alg-RCSL: deterministic env} with expectile regression and quantile regression
based on the RvS. 
Experiment results are shown in \Cref{tab:reusls_fixed_goal_0.99_additional}. We can conclude that $\algname$ outperforms the RvS framework across all fixed target RTG fraction ratios. 

\paragraph{$\algname$ with DT.} Experiment results of the DT-extension of our methods, DT-$\algname$, are shown in \Cref{tab:DT-antmaze}. We conclude that DT-$\algname$ outperforms vanilla DT across all settings.

\paragraph{$\algname$ incorporating dyamic programming.} Experiment results of DP-$\algname$, which is a hybrid method incorporating dynamic programming components, are shown in \Cref{tab:QT-antmaze}. We conclude that DP-$\algname$ significantly improves DT-$\algname$. DP-$\algname$ outperforms QT on the umaze environment, but is slightly worse on umaze-diverse and medium diverse environments. This is well expected as QT shares a key feature with our reinforced RCSL framework—namely, the use of an optimal conditioning function.

\begin{table}[t!]
\caption{{Normalized score on D4RL Antmaze for DT, QT, DT-$\algname$ and DP-$\algname$ with expectile ($\alpha$ = 0.99) and quantile ($\alpha$ = 0.99) regression respectively. We report the mean and standard deviation of the normalized score for five seeds.}}
\center
\label{tab:DT-antmaze}
\begin{tabular}{cccc}
\toprule
\textbf{Dataset}        & \textbf{DT-$\algname$-Expectile} & \textbf{DT-$\algname$-Quantile} & \textbf{DT}  \\ \midrule
umaze          & 72.8$\pm$3.90           & 71.4$\pm$3.51          & 59.2   \\
umaze-diverse  & 63.2$\pm$6.06           & 62.4$\pm$4.10          & 53   \\
medium-play    & 2.6$\pm$1.14            & 1.4$\pm$0.55           & 0     \\
medium-diverse & 2.4$\pm$1.52            & 3.4$\pm$1.14           & 0         \\ \midrule
Total          & {\bf 141.0}                   & 138.6                  & 112.2   \\ \bottomrule
\end{tabular}
\end{table}

\begin{table}[t!]
\caption{{Normalized score on D4RL Antmaze for QT, DP-$\algname$ with expectile ($\alpha$ = 0.99) and quantile ($\alpha$ = 0.99) regression respectively. We report the mean and standard deviation of the normalized score for five seeds.}}
\center
\label{tab:QT-antmaze}
\begin{tabular}{cccc}
\toprule
\textbf{Dataset} & \textbf{DP-$\algname$-Expectile} & \textbf{DP-$\algname$-Quantile} & \textbf{QT} \\ \midrule
umaze          &       97.8$\pm$3.03                  &      97.4$\pm$2.88                  &  96.7$\pm$4.7  \\
umaze-diverse  &       84.6$\pm$5.08                  &       88.2$\pm$8.14                 &  96.7$\pm$4.7  \\
medium-diverse &          50.2$\pm$7.08               &           48.4$\pm$4.72             &  59.3$\pm$0.9  \\ \midrule
Total          &           232.6              &            234            &   {\bf 252.7} \\ \bottomrule
\end{tabular}
\end{table}

\section{Multi-Step vs. Single-Step Stitching}
\label{sec:multi-step}
We have so far demonstrated that by incorporating a dataset-dependent optimal conditioning function $f^*$, our algorithm, $\algname$, provably converges to the optimal stitched policy $\pi_\beta^{\star}$, which outperforms classical RCSL, as verified by our experiments. However, unlike dynamic programming-based algorithms, which can converge to the optimal policy $\pi^{\star}$ independently of the specific dataset, the relationship between $\pi_\beta^{\star}$ and $\pi^{\star}$ still remains unclear in $\algname$. In this section, we extend the notion of the optimal stitched policy to explore this relationship further.

\noindent\textbf{Multi-step in-distribution optimal RTG.} 
We introduce a multi-step RTG relabeling procedure to enhance the $\algname$ framework based on the in-distribution optimal RTG function $f^{\star}$. This relabeling scheme iteratively predicts the optimal RTG from the current state. For any trajectory  $\tau = (s_1,a_1,g_1,s_2,a_2,g_2,\dots,s_H,a_H,g_H)$
in the feasible set $T_{\beta}$, we relabel the RTGs in a backward fashion for $k \geq 1$ passes.
We define {\it one-pass} of relabeling as the whole procedure of relabeling from the last stage to the first stage.

For the ease of discussion, we denote $\tilde{\tau}^0 = (s_1,a_1,\tilde{g}^0_1,s_2,a_2,\tilde{g}^0_2,\cdots,s_H,a_H,\tilde{g}^0_H) = \tau$, and $\tilde{T}^0_{\beta} = T_{\beta}$. 
Then for any $k\geq1$, suppose we have 
$\tilde{T}^{k-1}_{\beta}$, we define the feasible conditioning function set after $k-1$ passes relabeling as:
\begin{align*}
    \tilde{\cF}^{k-1}_{\beta} &:= \{f_{k-1} : \cS \times [H] \rightarrow \mathbb{R} \mid \forall (s, h) \in \text{dom}(f_{k-1}), \\
    &\quad \exists \, \tilde{\tau}^{k-1} \in \tilde{T}^{k-1}_{\beta} \, \text{s.t.} \, s_h = s \, \text{and} \, f_{k-1}(s, h) = \tilde{g}^{k-1}_h \},
\end{align*}
where $\text{dom}(f_{k-1})$ is the domain of $f_{k-1}$. 
At any stage $h \in [H]$, recall the feasible set of states is $\cS_h^{\beta} := \{s \in \cS \mid d_h^{\beta}(s) > 0\}$. For any feasible state $s \in \cS_h^{\beta}$, the local feasible conditioning function set after one-pass of relabeling is defined as:
\begin{align*}
   \tilde{\cF}_{\beta}^{k-1} (s, h) &:= \{f_{k-1} : \cS \times [H] \rightarrow \mathbb{R}\mid  f_{k-1} \in \tilde{\cF}^{k-1}_{\beta} \text{and} \, (s, h) \in \text{dom}(f_{k-1})\}.
\end{align*}
We then define the optimal conditioning function after $k-1$ passes of relabeling as:
\begin{align*}
    f_{k-1}^{\star}(s, h) := \argmax_{f \in \cF_{\beta}(s, h)} f_{k-1}(s, h),
\end{align*}
where $f_{k-1}^{\star}(s, h)$ represents the in-distribution optimal RTG from $(s, h)$ in the relabeled feasible set.
Given the trajectory set $\tilde{T}^{k-1}_{\beta}$ as well as the correspondingly defined multi-step in-distribution optimal RTG function $f^{\star}_{k-1}(s,a)$, the $k$-th pass of relabeling proceeds as follows.
At the last stage $H$, set $\tilde{g}^{k}_H=\tilde{g}^{k-1}_H$. Starting from stage $H-1$, the return-to-go $\tilde{g}^{k-1}_h$ is recursively replaced as follows
\begin{align}
    \tilde{g}^{k}_h = \max\{r_h + f_{k-1}^{\star}(s_{h+1}, h+1), r_h + \tilde{g}^{k}_{h+1}\}.\label{def:multi}
\end{align}
The trajectory after $k$ passes of relabeling is denoted as $\tilde{\tau}^{k} = (s_1,a_1,\tilde{g}^{k}_1,s_2,a_2,\tilde{g}^{k}_2,\cdots,s_H,a_H,\tilde{g}^{k}_H)$, and the accordingly updated feasible set is denoted as $\tilde{T}^{k}_{\beta}$. Finally we define the multi-step $\algname$ policy, $\tilde{\pi}_{\beta}^{k,\star}$, corresponding to the $k$ passes relabeling process as
\begin{align}
\label{eq: multi-step policy}
    \tilde{\pi}_{\beta}^{k,\star} := \tilde{P}^k_{\beta} (\cdot|s,h,f^{\star}_k(s,h)),
\end{align}
where $\tilde{P}^k_{\beta}$ is the distribution on $\tilde{T}^k_{\beta}$
induced by the behavior policy $\beta$ and $k$ passes relabeling.

We compare our key relabeling step \eqref{def:multi} with dynamic programming. At first glance, \eqref{def:multi} resembles the classical Bellman-type update, where for any state $s$ and action $a$, the optimal Q-function satisfies:  
\[
Q_h^{\star}(s,a) = r_h(s,a) + \mathbb{E}_{s'} V_{h+1}^{\star}(s'),
\]
which involves summing the immediate reward and the expected future return. However, a key distinction is that dynamic programming requires $V_{h+1}^{\star}$ to be the optimal value function, which is not directly obtainable unless we iteratively apply the Bellman equation to $Q_{h+1}^{\star}$ down to the final stage $H$. In contrast, our approach in \eqref{def:multi} relies solely on an achievable quantity, $f^{\star}_{k-1}$, which is retained from the $(k-1)$-th relabeling. This distinction makes our relabeling method a natural extension of RCSL towards dynamic programming-based algorithms.

\noindent\textbf{Theoretical guarantee.} We have the following theorem that suggests the multi-step relabeling scheme endows the $\algname$ with the capability of  `deep stitching': with a sufficient number of relabeling passes, $\algname$ utilizing return-to-go relabeling is guaranteed to achieve the optimal policy. To see this, let $\Pi_\beta = \{\pi|\forall (s,h)\in\cS\times[H], \pi(\cdot|s,h) \ll \beta(\cdot|s,h)\}$\footnote{For two distributions $P$ and $Q$, $P \ll Q$ means $P$ is absolutely continuous w.r.t. Q.} be the set of policies that are covered by the behavior policy $\beta$, defined the optimal value function covered by $\beta$ as $V_1^{\star,\beta}(s) = \max_{\pi\in\Pi_{\beta}}V_1^{\pi}(s)$. Then we have the following theorem.

\begin{theorem}
\label{th:best stitching}
Under deterministic environments, we have $J(\tilde{\pi}_{\beta}^{H-1,\star}) = \EE_{s\sim\rho}V_1^{\star,\beta}(s)$, where $\tilde{\pi}_{\beta}^{H-1,\star}$ is the $H-1$ step $\algname$ policy defined in \eqref{eq: multi-step policy} with $k=H-1$.
\end{theorem}

\Cref{th:best stitching} establishes that after $k = H-1$ relabeling passes, $\algname$ recovers the optimal stitched trajectory, akin to dynamic programming-based methods such as CQL \citep{kumar2020conservative}. Moreover, it implies that the relabeling process enhances the worst-case performance of $\algname$. Since $\algname$ is guaranteed to recover at least the best $k$-step stitched trajectory from the initial state, increasing $k$ further strengthens this guarantee.

\section{Conclusion}
\label{sec:conclusion}
We explore methods to provably enhance RCSL for effective trajectory stitching in offline datasets. To this end, we introduce $\algname$, which leverages a in-distribution optimal RTG quantity. We show that $\algname$ can learn the in-distribution optimal stitched policy, surpassing the best policy achievable by standard RCSL. Furthermore, we provide a theoretical analysis of $\algname$ and its variants.
Comprehensive experiment results demonstrate the effectiveness of the $\algname$ framework.

A notable limitation of the RCSL-type algorithms is that they can fail in stochastic environments. Specifically, \Cref{thm:RCSL with general function approximation} shows that the R2CSL algorithm can effectively recover the underlying objective policy $\pi_{\beta}^{\star}$, which is defined by the stochastic environment and the behavior policy. Prior works \citep{eysenbach2022imitating, brandfonbrener2022does} suggest that this objective policy $\pi_{\beta}^{\star}$ can be arbitrarily suboptimal, and we note that this is a  fundamental limitation of RCSL-style algorithms. It remains an open problem to theoretically address this limitation based on the RCSL framework.

\bibliography{reference}
\bibliographystyle{plainnat}

\newpage
\appendix

\section{Proof of Theorems}
In this section, we provide proofs of the theorems in the main text.
\subsection{Proof of \Cref{thm:superiority of the optimakl stitched policy}}
By Corollary 2 of \cite{brandfonbrener2022does}, we have $J(\pi_f^{\text{RCSL}}) = E_{s\sim\rho}[f(s,1)], \forall f\in\cF^{\text{Cst}}_{\beta}$. By the definition of $f$, we know that $f(s,1)$ is the return-to-go of a trajectory $\tau$ starting with $s_1=s$. Then $E_{s\sim\rho}[f(s,1)]$ is the weighted average of trajectories corresponding to $f$. In order to show $J(\pi^{\star})\geq J(\pi_f^{\text{RCSL}})$, we argue in the following that for any $s\in\cS$, $V_1^{\pi^{\star}}(s)$ corresponds to the weighted average of return-to-go of a set of stitched trajectories with $s_1=s$. And the stitched trajectories are no worse than any other trajectory in $T_{\beta}$ in terms of the cumulative reward (initial return-to-go). Thus we have 
    \begin{align*}
        J(\pi^{\star}) = \EE_{s\sim\rho}V_1^{\pi^{\star}}(s)\geq \EE_{s\sim\rho}V_1^{\pi^{\text{RCSL}}_f}(s) =J(\pi_f^{\text{RCSL}}), \forall f\in\cF^{\text{Cst}}_{\beta}.
    \end{align*}

    Denote $\tilde{f}=\argmax_{f\in\cF^{\text{Cst}}_{\beta}}f(s,1)$,
    we only need to show that the return-to-go of the  stitched trajectory is no worse than $\tilde{f}(s,1)$. In particular, $\tilde{f}$ corresponds to a set of trajectories with the biggest cumulative reward (initial return-to-go) starting with $s_1=s$. We denote this set as $\tilde{T}_1(s)$. At stage $h=1$, $\pi^{\star}$ would choose action according to the distribution of the first action $a_1$ in $\tilde{T}_1(s)$. The action $a_1^{\star}$ chosen by $\pi^{\star}$ at the first stage would result in a subset $\tilde{T}_2(s_2(s,a_1^{\star}))$ containing sub-trajectories (trajectory starting from the middle stage) that starts with $s_2(s,a_1^{\star})$ (the state transitioned from $s$ by taking $a_1^{\star}$ at the first stage), with the same return-to-go $\tilde{f}(s,1)-r(s_1,a_1^{\star})$.

    Starting from the second stage, $\pi^{\star}$ starts stitching the performance of different trajectories. Simply denoting $s_2(s,a_1^{\star})$ as $s_2^{\star}$, then we recall that 
    \begin{align*}
        f^{\star}(s_2^{\star},2)=\argmax_{f\in\cF_{\beta}(s_2^{\star},2)}f(s_2^{\star},2).
    \end{align*}
    By the equivalence of the conditioning function and the set of trajectories, it basically means that we choose another subset of trajectories $T_2^{\star}(s_2^{\star})$, starting from $s_2^{\star}$ but with better return-to-go compared to $\tilde{T}_s(s_2^{\star})$, to stitch. And then $\pi^{\star}$ chooses action according to the distribution of $a_2$ in $\tilde{T}_s(s_2^{\star})$. So on and so forth, it is trivial that the trajectory induced by $\pi^{\star}$ is better than the trajectory in $\tilde{T}_1(s)$. Thus, we complete the proof.

\subsection{Proof of \Cref{thm:finite sample guarantee for reinforced RCSL - partial coverage}}

Under the \Cref{assumption: regular,assumption: data1}, we known that when the sample size is large enough, the maximum returns will be included in the dataset with high probability, and thus $\hat{f}^{\star}(s_h,h) = f^{\star}(s_h,h)$.
In particular, we have
\begin{align*}
P\big(s,h,g_h=f^{\star}(s,h) \big)\geq d_{\min}^{\beta}\cdot\tilde{c}.
\end{align*}
Then for any $(s_h,h)$, we want
\begin{align*}
    P\big(\forall k\in T_{\cD}(s_h), g_h^k\neq f^{\star}(s_h,h) \big) \leq \big(1-d_{\min}^{\beta}\cdot\tilde{c}\big)^N \leq \frac{\delta}{SH},
\end{align*}
and this leads to 
\begin{align*}
    N\geq \frac{\log\frac{SH}{\delta}}{\log(1-d_{\min}^\beta\cdot \tilde{c})}.
\end{align*}
Then by union bound, when $N>\log(SH/\delta)/\log(1-d_{\min}^\beta\cdot \tilde{c})$, with probability at least $1-\delta$, the following event holds
\begin{align*}
    \cE =\{\forall (s_h,h), \exists k\in T_{\cD}(s_h),~s.t.~g_h^k= f^{\star}(s_h,h)\}.
\end{align*}
Under the event $\cE$, we have $\hat{f}^{\star}(s_h,h) = f^{\star}(s_h,h)$.

Second, we bound the suboptimality the event $\cE$. By the definition of the $J(\pi)$, we have
    \begin{align*}
        J(\pi_\beta^{\star})  - J(\hat{\pi}_{\cD}^{\star})  &= H\Big[\EE_{P}^{\pi_\beta^{\star}}[r(s,a)] - \EE_P^{\hat{\pi}_{\cD}^{\star}}[r(s,a)] \Big]\leq H\big \|d^{\star,\beta} - d^{\star,\cD}\big\|_1,
    \end{align*}
    where $d^{\star,\beta}$ and $d^{\star,\cD}$ are the occupancy measures on state induced by $\pi_\beta^{\star}$, and  $\hat{\pi}_{\cD}^{\star}$.
    By \Cref{lemma:recursive formula}, we have
    \begin{align*}
        &J(\pi_\beta^{\star})  - J(\hat{\pi}_{\cD}^{\star})\\
        & \leq 2H\sum_{h=1}^H \EE_{s\sim d_h^{\star,\beta}} \big[\text{TV}\big(\pi_\beta^{\star}(\cdot|s,h)||\hat{\pi}_{\cD}^{\star}(\cdot|s,h)\big)\big]\\
        & = 2H \sum_{h=1}^H \EE_{s\sim d_h^{\star,\beta}}\big[\text{TV}\big(P_\beta(\cdot|s,h,f^{\star}(s,h))|| \hat{\pi}(\cdot|s,h,\hat{f}^{\star}(s,h))\big) \big]\\
        & = 2H \sum_{h=1}^H \EE_{s\sim d_h^{\star,\beta}}\big[\text{TV}\big(P_\beta(\cdot|s,h,f^{\star}(s,h))|| \hat{\pi}(\cdot|s,h,f^{\star}(s,h))\big) \big]\\
        & = 2H \sum_{h=1}^H \EE_{s\sim d_h^{\star,\beta}}\Big[\int_a\big|P_{\beta}(a|s,h,f^{\star}(s,h)) - \hat{\pi}(\cdot|s,h,f^{\star}(s,h)) \big| \Big]\\
        & =2H \sum_{h=1}^H \EE_{s\sim d_h^{\star,\beta}}\Big[\frac{P_{\beta}(f^{\star}(s,h)|s,h)}{P_{\beta}(f^{\star}(s,h)|s,h)}\int_a\big|P_{\beta}(da|s,h,f^{\star}(s,h)) - \hat{\pi}(da|s,h,f^{\star}(s,h)) \big| \Big]\\
        & =2H \sum_{h=1}^H \EE_{s\sim d_h^{\star,\beta}}\Big[\frac{P_{\beta}(f^{\star}(s,h)|s,h)}{P_{\beta}(f^{\star}(s,h)|s,h)}\int_a\big|P_{\beta}(da|s,h,f^{\star}(s,h)) - \hat{\pi}(da|s,h,f^{\star}(s,h)) \big| \Big]\\
        & \leq \frac{c^{\star}_{\beta}H}{\tilde{c}}\sum_{h=1}^H \EE_{s\sim d_h^{\beta}}\Big[{P_{\beta}(f^{\star}(s,h)|s,h)}\int_a\big|P_{\beta}(da|s,h,f^{\star}(s,h)) - \hat{\pi}(da|s,h,f^{\star}(s,h)) \big| \Big]\\
        &\leq \frac{c^{\star}_{\beta}H}{\tilde{c}}\sum_{h=1}^H \EE_{s\sim d_h^{\beta}}\Big[\int_g P_{\beta}(dg|s,h)\int_a\big|P_{\beta}(da|s,h,g) - \hat{\pi}(da|s,h,g) \big| \Big]\\
        & = \frac{2c^{\star}_{\beta}H}{\tilde{c}}\sum_{h=1}^H\EE_{s\sim d_h^{\beta},g\sim P_\beta|s,h}\big[\text{TV}\big(P_{\beta}(\cdot|s,h,g) || \hat{\pi}(\cdot|s,h,g) \big) \big]\\
        & \leq \frac{c^{\star}_{\beta}H^2}{\tilde{c}}\sqrt{2L(\hat{\pi})},
    \end{align*}
    where the second equation holds under the event $\hat{f}^{\star}(s_h,h) = f^{\star}(s_h,h)$, the second inequality holds by assumption (1), and 
    the last step follows from the Pinsker's inequality. Next, for any  $(\pi)\in \Pi$, we write $L(\pi) = \bar{L}(\pi)-H_{\beta}$,  where $H_{\beta} = -\EE[\log P_{\beta}(a|s,h,g)]$ and $\bar{L}(\pi) = -\EE[\log \pi(a|s,h,g)]$. Denoting $\pi^{\dagger}\in\argmin_{\pi\in\Pi}L(\pi)$, we have
    \begin{align*}
        L(\hat{\pi}) = L(\hat{\pi}) - L(\pi^{\dagger}) + L(\pi^{\dagger}) \leq \bar{L}(\hat{\pi})-\bar{L}(\pi^{\dagger}) + \delta_{\text{approx}}.
    \end{align*}
    Denote $\hat{L}$  as the empirical cross-entropy loss that is minimized by $\hat{\pi}$,  we have
    \begin{align*}
       \bar{L}(\hat{\pi}) - \bar{L}(\pi^{\dagger}) &=   \bar{L}(\hat{\pi})  - \hat{L}(\hat{\pi}) +  \hat{L}(\hat{\pi}) -\hat{L}(\pi^{\dagger}) + \hat{L}(\pi^{\dagger}) - \bar{L}(\pi^{\dagger})\\
       &\leq 2 \sup_{\pi\in\Pi}|\bar{L}(\pi) - \hat{L}(\pi)|.
    \end{align*}
    Under \Cref{assumption: regular}, we bound this using McDiarmid's inequality and a union bound. This completes the proof.

\subsection{Proof of \Cref{thm:RCSL with general function approximation}}

    By the definition of the $J(\pi)$, we have
    \begin{align*}
        J(\pi_\beta^{\star})  - J(\hat{\pi}_{\cD}^{\star})  &= H\Big[\EE_{P}^{\pi_\beta^{\star}}[r(s,a)] - \EE_P^{\hat{\pi}_{\cD}^{\star}}[r(s,a)] \Big]\leq H\big \|d^{\star,\beta} - d^{\star,\cD}\big\|_1,
    \end{align*}
    where $d^{\star,\beta}$ and $d^{\star,\cD}$ are the occupancy measures on state induced by $\pi_\beta^{\star}$, and  $\hat{\pi}_{\cD}^{\star}$.
    By \Cref{lemma:recursive formula}, we have
    \begin{align*}
        &J(\pi_\beta^{\star})  - J(\hat{\pi}_{\cD}^{\star})\\
        & \leq 2H\sum_{h=1}^H \EE_{s\sim d_h^{\star,\beta}} \big[\text{TV}\big(\pi_\beta^{\star}(\cdot|s,h)||\hat{\pi}_{\cD}^{\star}(\cdot|s,h)\big)\big]\\
        & \leq  {2c^{\star}_{\beta}H} \sum_{h=1}^H \EE_{s\sim d_h^{\beta}}\big[\text{TV}\big(P_\beta(\cdot|s,h,f^{\star}(s,h))|| \hat{\pi}(\cdot|s,h,\hat{f}^{\star}(s,h))\big) \big]\\
        & \leq  {2c^{\star}_{\beta}H} \sum_{h=1}^H \EE_{s\sim d_h^{\beta}}\big[\text{TV}\big(P_\beta(\cdot|s,h,f^{\star}(s,h))|| \hat{\pi}(\cdot|s,h,f^{\star}(s,h))\big)\\
        &\qquad + \text{TV}\big(\hat{\pi}(\cdot|s,h,f^{\star}(s,h))|| \hat{\pi}(\cdot|s,h,\hat{f}^{\star}(s,h))\big)  \big]\\
        & \leq  {2c^{\star}_{\beta}H} \sum_{h=1}^H \EE_{s\sim d_h^{\beta}}\big[\text{TV}\big(P_\beta(\cdot|s,h,f^{\star}(s,h))|| \hat{\pi}(\cdot|s,h,f^{\star}(s,h))\big) + \gamma\big|f^{\star}(s,h) - \hat{f}^{\star}(s,h)\big| \big]\\
        &\leq {2c^{\star}_{\beta}H} \sum_{h=1}^H \EE_{s\sim d_h^{\beta}}\big[\text{TV}\big(P_\beta(\cdot|s,h,f^{\star}(s,h))|| \hat{\pi}(\cdot|s,h,f^{\star}(s,h))\big) \big] \\
        &\qquad + {2c^{\star}_{\beta}H}\sum_{h=1}^H\gamma\sqrt{\EE_{s\sim d_h^{\beta}}\big(f^{\star}(s,h) - \hat{f}^{\star}(s,h)\big)^2}\\
        & \leq {2c^{\star}_{\beta}H}\sum_{h=1}^H \EE_{s\sim d_h^{\beta}}\big[\text{TV}\big(P_\beta(\cdot|s,h,f^{\star}(s,h))|| \hat{\pi}(\cdot|s,h,f^{\star}(s,h))\big) \big] + {2c^{\star}_{\beta}H^2\gamma}\sqrt{\text{Err}(N,\delta)}\\
        & \leq {2c^{\star}_{\beta}H} \sum_{h=1}^H \EE_{s\sim d_h^{\beta}}\Big[\int_a\big|P_{\beta}(a|s,h,f^{\star}(s,h)) - \hat{\pi}(\cdot|s,h,f^{\star}(s,h)) \big| \Big] + {2c^{\star}_{\beta}H^2\gamma}\sqrt{\text{Err}(N,\delta)}\\
        & ={2c^{\star}_{\beta}H} \sum_{h=1}^H \EE_{s\sim d_h^{\beta}}\Big[\frac{P_{\beta}(f^{\star}(s,h)|s,h)}{P_{\beta}(f^{\star}(s,h)|s,h)}\int_a\big|P_{\beta}(da|s,h,f^{\star}(s,h)) - \hat{\pi}(da|s,h,f^{\star}(s,h)) \big| \Big] \\
        &\qquad + {2c^{\star}_{\beta}H^2\gamma}\sqrt{\text{Err}(N,\delta)}\\
        & \leq \frac{c^{\star}_{\beta}H}{\tilde{c}}\sum_{h=1}^H \EE_{s\sim d_h^{\beta}}\Big[{P_{\beta}(f^{\star}(s,h)|s,h)}\int_a\big|P_{\beta}(da|s,h,f^{\star}(s,h)) - \hat{\pi}(da|s,h,f^{\star}(s,h)) \big| \Big] \\
        & \qquad + {2c^{\star}_{\beta}H^2\gamma}\sqrt{\text{Err}(N,\delta)}\\
        &\leq \frac{c^{\star}_{\beta}H}{\tilde{c}}\sum_{h=1}^H \EE_{s\sim d_h^{\beta}}\Big[\int_g P_{\beta}(dg|s,h)\int_a\big|P_{\beta}(da|s,h,g) - \hat{\pi}(da|s,h,g) \big| \Big] + {2c^{\star}_{\beta}H^2\gamma}\sqrt{\text{Err}(N,\delta)}\\
        & = \frac{2c^{\star}_{\beta}H}{\tilde{c}}\sum_{h=1}^H\EE_{s\sim d_h^{\beta},g\sim P_\beta|s,h}\big[\text{TV}\big(P_{\beta}(\cdot|s,h,g) || \hat{\pi}(\cdot|s,h,g) \big) \big] + {2c^{\star}_{\beta}H^2\gamma}\sqrt{\text{Err}(N,\delta)}\\
        & \leq \frac{c^{\star}_{\beta}H^2}{\tilde{c}}\sqrt{2L(\hat{\pi})} + {2c^{\star}_{\beta}H^2\gamma}\sqrt{\text{Err}(N,\delta)},
    \end{align*}
    where the last step follows from the Pinsker's inequality. 
    Next, for any  $(\pi)\in \Pi$, we write $L(\pi) = \bar{L}(\pi)-H_{\beta}$,  where $H_{\beta} = -\EE[\log P_{\beta}(a|s,h,g)]$ and $\bar{L}(\pi) = -\EE[\log \pi(a|s,h,g)]$. Denoting $\pi^{\dagger}\in\argmin_{\pi\in\Pi}L(\pi)$, we have
    \begin{align*}
        L(\hat{\pi}) = L(\hat{\pi}) - L(\pi^{\dagger}) + L(\pi^{\dagger}) \leq \bar{L}(\hat{\pi})-\bar{L}(\pi^{\dagger}) + \delta_{\text{approx}}.
    \end{align*}
    Denote $\hat{L}$  as the empirical cross-entropy loss that is minimized by $\hat{\pi}$,  we have
    \begin{align*}
       \bar{L}(\hat{\pi}) - \bar{L}(\pi^{\dagger}) &=   \bar{L}(\hat{\pi})  - \hat{L}(\hat{\pi}) +  \hat{L}(\hat{\pi}) -\hat{L}(\pi^{\dagger}) + \hat{L}(\pi^{\dagger}) - \bar{L}(\pi^{\dagger})\\
       &\leq 2 \sup_{\pi\in\Pi}|\bar{L}(\pi) - \hat{L}(\pi)|.
    \end{align*}
    Under \Cref{assumption: regular}, we bound this using McDiarmid's inequality and a union bound. This completes the proof.

\subsection{Proof of Hard Instances for $\algname$ with Expectile Regression}
\label{sec: hard instances for expectile regression}

We prove our claim that expectile regression leads to out-of-distribution returns as presented in \Cref{sec:practical}. We use the following toy example:
\begin{align*}
    &h=1 \quad ~h=2\quad ~h=3\\
    a^1 \quad &70(80) \quad 0(10) \quad ~10(10)\\
    a^2 \quad &65(81) \quad 15(16) \quad 1(1)\\
    a^3 \quad &40(75) \quad 20(35) \quad 15(15)
\end{align*}
In this toy example, $\cA=\{a^1,a^2,a^3\}$, $H=3$ and $\cS = \{s\}$.
The state is unique and remains unchanged across stages. At each stage, there are three actions that can be chosen. Each row represents a trajectory induced by implementing the action listed at the beginning. The number outside (inside) the parentheses are rewards (return-to-go). For RCSL, we can only simply choose $f$ to be 80, 81 or 75 as the conditioning return at the initial stage, and the `optimal' RCSL policy would choose $a^2$ at each stages. Although there are only three suboptimal trajectories, these suboptimal trajectories collectively cover a better trajectory. Let's see: at the first stage we choose $a^2$, at the second stage we choose $a^3$ and at the final stage we choose $a^3$. This leads to a trajectory 
\begin{align*}
   \overset{h=1}{a^2: 65(100)}  \rightarrow \overset{h=2}{a^3:20(35)} \rightarrow \overset{h=3}{a^3:15(15)}.
\end{align*}
This policy can ideally be inferred by the reinforced RCSL: during the inference process, at each stage, we set the conditioning function as the largest return-to-go corresponding to that stage. Specifically, at the first stage, the largest return-to-go, 81, comes from the second trajectory, thus we set $f_1=81$ and $\pi^{\text{RCSL}}(f_1,h=1)=a^2$; at the second stage, the largest return-to-go, 35, comes from the third trajectory, thus we set $f_2=35$ and $\pi^{\text{RCSL}}(f_2,h=2)=a^3$; at the third stage, the largest return-to-go, 15, comes from the third trajectory, thus we set $f_3=15$ and $\pi^{\text{RCSL}}(f_3,h=3)=a^3$. We highlight that in this toy example, the trajectories implicitly have overlap since the underlying state is unique and remains unchanged, and the dependence of the RCSL policy on the state is also omitted.

Consider policy learning, we set the behavior policy as the uniform distribution on the action space, $\beta = \text{Uniform}(\cA)$. 
Note that in the tabular MDP, the expectile regression is conducted at each state-action pair. Next, let's focus on the second stage. The return-to-go candidates are $\{10, 16, 35\}$, and our goal is to find the in-distribution optimal RTG, which is 35 at stage $h=2$ coming from the third trajectory
\begin{align}
\label{eq:the third trajectory in the toy example}
   \overset{h=1}{a^3: 40(75)}  \rightarrow \overset{h=2}{a^3:20(35)} \rightarrow \overset{h=3}{a^3:15(15)}.
\end{align}
However, due to the fact that the expectile regression uses $L_2$ loss, as long as the offline dataset includes a trajectory other than \eqref{eq:the third trajectory in the toy example}, the expectile regression with $\alpha<1$ would return a value less than 35 and may not being among the candidates $\{10, 16, 35\}$ (and thus being out-of-distribution). Then \Cref{alg-RCSL: deterministic env} with expectile regression fails at the second stage.  
On the other hand, if all trajectories in the offline dataset is the trajectory \eqref{eq:the third trajectory in the toy example}, then the offline dataset does not contain any information about the second trajectory, which is a necessary component of the optimal stitched trajectory. In this case, \Cref{alg-RCSL: deterministic env} with expectile regression will fail in the first stage. In conclusion, \Cref{alg-RCSL: deterministic env} with expectile regression will definitely fail in the first or the second stage. This completes the proof.

\subsection{Proof of \Cref{thm:theoretical guarantee for reinforced RCSL with quantile regression}}

    We only need to show that with probability at least $1-2\delta$, for any $(s,h)$, we have $P(X_h(f^{\star}(s,h)) \geq N_h^2\cdot\tilde{c}/2)$. Then setting $\alpha > 1-\tilde{c}/2$, the $\alpha$-quantile is exactly $f^{\star}(s,h)$.

    By the Hoeffding inequality and the assumption that $P_{\beta}(g_h=f^{\star}(s,h)|s_h=s)\geq \tilde{c}$, we have
    \begin{align*}
        P(X_h(f^{\star}(s,h))\leq N_h^s\cdot \tilde{c} - t)\leq P(X_h(f^{\star}(s,h))\leq P_{\beta}(g_h=g|s_h=s) - t)\leq \exp(-2t^2/N_h^s).
    \end{align*}
    Let $t=\tilde{c}N_h^s/2$, we have
    \begin{align*}
        P(X_h(f^{\star}(s,h))\leq N_h^s\cdot \tilde{c}/2)\leq \exp(-\tilde{c}^2N_h^s/2)\leq \frac{\delta}{2SH}.
    \end{align*}
    This leads to $N_h^s\geq 2\log(2SH/\delta)/\tilde{c}^2$.
    Note that $N_h^s$ is also a random variable, next we study the condition under which $N_h^s\geq 2\log(2SH/\delta)/\tilde{c}^2$ holds with high probability.
    In particular, we have
    \begin{align*}
        P(N_h^s \leq N\cdot d_{\min}^{\beta} - t)\leq P(N_h^s\leq N\cdot d_h^{\beta}(s) - t)\leq \exp(-2t^2/N).
    \end{align*}
    Let $t=N\cdot d_{\min}^{\beta}/2$ and make the above inequality be less than $\delta/2SH$, we derive that when 
    \begin{align*}
        N\geq \frac{2}{d_{\min}^{\beta,2}}\log\frac{2SH}{\delta},
    \end{align*}
    we have 
    \begin{align*}
        P(N_h^s\geq N\cdot d_{\min}^{\beta}/2)\geq 1-\delta/2.
    \end{align*}
    Moreover, we want $N\cdot d_{\min}^{\beta}/2 \geq 2\log(2SH/\delta)/\tilde{c}^2$, which leads to
    \begin{align*}
        N\geq\frac{4}{\tilde{c}^2d_{\min}^{\beta}}\log\frac{2SH}{\delta},
    \end{align*}
    thus we have 
    \begin{align*}
        P\Big(N_h^s\geq  2\log(2SH/\delta)/\tilde{c}^2\Big) \geq P\Big(N_h^s\geq  \frac{N\cdot d_{\min}^{\beta}}{2}\Big)\geq 1-\frac{\delta}{2}.
    \end{align*}
    By union bound, when 
    \begin{align*}
    N\geq \max\Big\{\frac{2}{d_{\min}^{\beta,2}}\log\frac{2SH}{\delta}, \frac{4}{\tilde{c}^2d_{\min}^\beta}\log\frac{2SH}{\delta}\Big\},
\end{align*}
with probability at least $1-\delta$, we have $P(X_h(f^{\star}(s,h)) \geq N_h^2\cdot\tilde{c}/2)$. This completes the proof.

\subsection{Proof of \Cref{th:best stitching}}
\label{sec: proof of th:best stitching}

     Define the optimal Q-function covered by the behavior policy $\beta$ as $Q_h^{\star,\beta}(s,a) = \max_{\pi\in\Pi_{\beta}}Q_h^{\pi}(s,a)$. By bellman optimality equation, we know that $V_h^{\star,\beta}(s) = \max_{a\in\cA, \beta(a|s)>0}Q_h^{\star,\beta}(s,a)$.
     We prove \Cref{th:best stitching} by showing that the new label $\tilde{g}^{H-1}_h$ is $Q_h^{\star,\beta}(s_h,a_h)$, and the conditioning function $\tilde{f}^{\star}_{k}$ serves as the maximum operator in defining the value function.

    To show this, we first answer the following easier question: 
    \begin{center}
        {\it After one-pass of the relabeling procedure, what does the new RTG in the trajectory mean? }
    \end{center}   
    Equivalently, we would like to figure out what does the return-conditioned policy based on modified data aim for? With the original trajectory $\tau=(s_1,a_1,g_1,\cdots, s_H,a_H,g_H)$, next we delve into the relabeling process. Starting from the last stage $H$, we have $\tilde{g}^1_H=g_H$, which is the reward obtained by adopting $a_H$ at $s_H$. Then the trajectory becomes $\tau=(s_1,a_1,g_1,\cdots, s_H,a_H,\tilde{g}^1_H)$. 
    \paragraph{Stage $H-1$ in the first pass  relabeling.}
    Moving one stage backward, we perform some real relabeling at stage $H-1$.
    \begin{align*}
        \tilde{g}_{H-1} = \max\{\underbrace{r_{H-1} + f^{\star}(s_H, H)}_{\text{I}}, \underbrace{r_{H-1} + \tilde{g}^1_H}_{\text{II}}\}.
    \end{align*}
    From now on, at each stage, we answer the following two questions to find common patterns. 
    \begin{align}
        &\text{Q1: what are the term I and term II?}\label{eq:multi-step-question1}\\
        &\text{Q2: what does the maximization operation mean?}\label{eq:multi-step-question2}
    \end{align}   
    We first focus on Q1. The relabeling consists of two parts. Term I: the current reward + the in-distribution optimal RTG in the original data, and term II: the original RTG $g_{H-1}$. 
    In both term I and term II, there are two parts: the current reward and the future cumulative reward. In term I, the future cumulative reward is the maximum possible RTG at $s_H$ achieved by $\beta$. So term I represents the cumulative reward we can get at $s_{H-1}$ if we take $a_{H-1}$ at the current stage and then take the action associated to $f^{\star}(s_H,H)$. Term II represents the cumulative reward we can get at $s_{H-1}$ if we take $a_{H-1}$ at the current stage and then take original action in the trajectory $\tau$, which is $a_H$. This answers Q1 in \eqref{eq:multi-step-question1}. We then proceed to answer Q2. Note that at stage $H-1$, there are two cases can happen: (i) $\text{term I} = \text{term II}$, and (ii) $\text{term I} > \text{term II}$. 
    
    For case (i), the tie means that $a_H$ is the action that achieves the maximum RTG, which is the maximum reward at $s_H$. 
    
    For case (ii), term I $>$ term II means that there is a better action, which is associated with $f^{\star}(s_H,H)$, and we better follow that action in at $s_H$.
    Thus, it is clear that at stage $H-1$, if case (ii) happens, we actually would perform {\it one-step stitching} (because the action associated with $f^{\star}(s_H,H)$ differs from $a_H$, in order words it comes from other trajectories instead of $\tau$ itself), and the stitching technically happens in the next stage $H$. This answers Q2 in \eqref{eq:multi-step-question2}. 

    After relabeling $g_{H-1}$, we have 
    \begin{align*}
        \tau = \{s_1,a_1,g_1, \cdots, s_{H-2},a_{H-2}, g_{H-2},s_{H-1},a_{H-1},\tilde{g}^1_{H-1}, s_H,a_H,\tilde{g}^1_H\}.
    \end{align*}
    Before we move on, let's stop for a while to consider a question: What does $\tilde{g}_{H-1}$ represent? Based on the answer to Q1 and Q2 above, we can tell that $\tilde{g}_{H-1}$ represents the return-to-go we can get if we follow the action in the trajectory $a_{H-1}$ at $s_{H-1}$ and then take the optimal action supported by $\beta$ at $s_H$, i.e., the optimal RTG we can get after following $a_{H-1}$.

\paragraph{Stage $H-2$ in the first pass relabeling.} 
    One stage backward, let's consider the stage $H-2$
    \begin{align*}
        \tilde{g}_{H-2} = \max\{\underbrace{r_{H-2} + f^{\star}(s_{H-1}, H-1)}_{\text{I}},\underbrace{r_{H-2}+\tilde{g}^1_{H-1}}_{\text{II}} \}.
    \end{align*}
    Term I is the current reward plus maximum RTG at $s_{H-1}$ achieved by $\beta$. Term II is the current reward plus optimal RTG after following $a_{H-1}$ at $s_{H-1}$.  As for the max-operator, the comparison between term I and term II is to check if there is a better action for stage $H-1$ at $s_{H-1}$. 
    
    To see this, we analyze the following three possible cases: (i) $\text{term I} = \text{term II}$, (ii) $\text{term I} > \text{term II}$, and (iii) $\text{term I} < \text{term II}$. 
    \subparagraph{Case (i):}  Tie means that following action $a_{H-1}$ at $s_{H-1}$ would be fine.
    \subparagraph{Case (ii):}  There is a better choice of action than $a_{H-1}$ which leads to a larger cumulative reward (in-distribution optimal RTG at $s_{H-1}$).
    \subparagraph{Case (iii):} We better follow $a_{H-1}$ at $s_{H-1}$ because after checking the RTGs of all possible trajectories with $s_{H-1}$ induced by $\beta$, no one is larger than $\tilde{g}_{H-1}$, which is the cumulative reward of a sub-trajectory starting with $(s_{H-1},a_{H-1})$ and possibly involving one-time stitching at $s_H$.

    To sum up, the maximization operation actually gives us a chance to figure out if the feasible set suggests a better choice of action at $s_{H-1}$, which could lead to a new trajectory. Note that the term `better' is in the sense that it leads to a in-distribution RTG that is larger than the cumulative reward of a sub-trajectory starting with $(s_{H-1},a_{H-1})$ and possibly involving one-time stitching at $s_H$. We highlight that the in-distribution RTG is the RTG corresponding to a sequence of actions in some original trajectory involved in $T_{\beta}$, which does not involve stitching.
    If (ii) happens, choosing term I as the relabeled RTG means that we perform one-step stitching at $s_{H-1}$ by adopting an action other than the original action $a_{H-1}$ in the trajectory, and forget about the sub-trajectory starting with $(s_{H-1},a_{H-1})$ and possibly involving one-time stitching at $s_H$. It will become clear later that the relabeling process is effectively performing a form of shallow stitching/planning, as the new label only remembers an one-time stitching.

\paragraph{Stage $H-3$ in the first pass relabeling.}
    Next, we focus on $\tilde{g}^1_{H-3}$, where 
     \begin{align*}
        \tilde{g}_{H-3} = \max\{\underbrace{r_{H-3} + f^{\star}(s_{H-2}, a_{H-2})}_{\text{I}}, \underbrace{r_{H-3}+\tilde{g}^1_{H-2}}_{\text{II}}\}.
    \end{align*}
    At $s_{H-3}$, we follow $a_{H-3}$. Then at the next stage $H-2$ we either follow the action associated to $f^{\star}(s_{H-2}, H-2)$, or follow the original action $a_{H-2}$, which could provide a better future stitching. So choosing term I/term II basically means stitching at the next stage or stitching maybe in the further future (after the next stage). 
    
    \paragraph{A conclusion of one-pass relabeling.}
    According to the analysis above, the new label $\tilde{g}_h$ is the cumulative reward achieved by following $a_h$ at the current stage and then following the best one-time stitching trajectory afterwards. One pass of the relabeling process incorporates information about the future best one-time stitching into the new RTG label.    
    
    Starting from $\tilde{\tau} = (s_1,a_1,\tilde{g}^1_1,\cdots,s_H, a_H,\tilde{g}^1_H)$, the second pass of the relabeling process endows the label $\tilde{g}_h^2$ information about the future best two-time stitching trajectory. This is exactly the dynamic programming in deterministic environments, and after $H-1$ passes, we have  $\tilde{g}^{H-1}_h = Q^{\star,\beta}_h(s_h,a_h)$.

Lastly, the conditioning function $\tilde{f}^{\star}_k$ is defined based on $\tilde{\tau}^{H-1}$. By definition, $\tilde{f}^{\star}_k$ is the in-distribution optimal RTG label
    \begin{align*}
        \tilde{f}^{\star}_k(s_h,h)=
        \max_{a\in\cA, \beta(a|s_h)>0}\tilde{g}^{H-1}_h(s_h,a) = \max_{a\in\cA, \beta(a|s_h)>0}Q^{\star,\beta}_h(s_h,a_h),
    \end{align*}
    which identifies the best action that achieves the optimal value function at $s_h$.
    Thus, the reinforced RCSL policy $\tilde{\pi}_{\beta}^{H-1,\star}$
    recovers the optimal policy. This completes the proof.

\section{The Auxiliary Lemmas}
\begin{lemma}[Lemma 1 of \cite{brandfonbrener2022does}]
\label{lemma:recursive formula}
Let $d^\pi$ refer to the marginal distribution of $P^\pi$ over states only. For any two policies $\pi$ and $\pi'$, we have 
\begin{align*}
    \|d^\pi - d^{\pi'}\|_1\leq 2\sum_{h=1}^H\EE_{s\sim d_h^\pi}\big[\text{TV}\big(\pi(\cdot|s,h)||\pi'(\cdot|s,h)\big)\big].
\end{align*}
\end{lemma}
\begin{proof}
    The proof largely follows that of Lemma 1 in \cite{brandfonbrener2022does}. By (36) in the poof of Lemma 1 in \cite{brandfonbrener2022does}, we have 
    \begin{align*}
        \|d^\pi - d^{\pi'}\|_1 &\leq \frac{1}{H}\sum_{h=1}^H\Delta_h \leq \frac{1}{H}\sum_{h=1}^H\sum_{j=1}^{h-1}\delta_j \leq H\frac{1}{H}\sum_{h=1}^H\delta_h = 2\sum_{h=1}^H\EE_{s\sim d_h^\pi}\big[\text{TV}\big(\pi(\cdot|s,h)||\pi'(\cdot|s,h)\big) \big],
    \end{align*}
    where $\Delta_h =\|d_h^{\pi}-d_h^{\pi'}\|_1$ and $\delta_h=2\EE_{s\sim d_h^{\pi}}\big[\text{TV}(\pi(\cdot|s, h)||\hat{\pi}'(\cdot|s, h) \big]$.
\end{proof}

\section{Experiments Details}
\label{sec: Experiments Details}

D4RL \citep{fu2020d4rl} is an offline RL benchmark which provide the pre-collected dataset such as Gym-MuJoCo, AntMaze, and Kitchen. In this work, we evaluate our work in the Gym-MuJoCo with the medium, medium-replay, and medium-expert three-level results. For AntMaze, we evaluate our methods on umaze, umaze-diverse, medium-play and medium diverse. All experiments are based on five seeds: 1000, 2000, 3000, 4000 and 5000.

\subsection{Implementation and Hyperparameters}
In this section, we provide the implementation details and hyperparameters of our $\algname$-Expectile and  $\algname$-Quantile. 

\paragraph{Implementation of $\algname$ with RvS} To validate our proposed algorithm, we follow the settings in \cite{emmons2022rvs}. RvS leverages the vanilla RTG as the conditioning function to predict the optimal action.  We keep the training stage as the same as RvS and modify the inference stage. During the inference, we first pre-train a condition function by expectile regression and quantile regression using MLP. Then we use the pre-trained condition function to predict the max-return given the current state in the inference. We use the RTG prediction to guide the action selection in RvS.

We evaluate our proposed $\algname$ method using the RvS framework on the D4RL benchmark \citep{fu2020d4rl}. In the RvS framework, achieving optimal performance during inference requires searching for the target RTGs \citep{emmons2022rvs}. However, this process is often impractical in real-world scenarios. To address this, we select three appropriate target RTG fraction ratios as $0.7,0.9,1.1$ to guide the RvS instead, which means the initial target RTG will be 
\begin{align}
  \text{RTG}_{\text{inital}} = (\text{RTG}_{\max}-\text{RTG}_{\min}) * \text{fraction} + \text{RTG}_{\min},
\end{align}  
where $\text{RTG}_{\max}$ and $\text{RTG}_{\min}$ are the maximum RTG and minimum RTG from the random policy and expert policy in this specific environment respectively. Unlike the RvS approach, our method does not require additional searches for initial target RTGs, as these values are predicted by our pre-trained condition function.

We follow the same implementation as the RvS during the training stage shown in \Cref{tab:hyper_rvs}. We also illustrate our MLP condition function parameters in \Cref{tab:hyper_mlp}.

We run these experiments on the RTX2080Ti for around 8 hours for each setting.

\begin{table}[ht]
\caption{Hyperparameters in RvS.}
\centering
\label{tab:hyper_rvs}
\begin{tabular}{cc}
\toprule
\textbf{Hyperparameter} & \textbf{Value}                                         \\ \midrule
Hidden layers           & 2                                                     \\
Layer width             & 1024                                                              \\
Nonlinearity            & ReLU                                                             \\
Learning rate           & $1 \times 10^{-3}$                                            \\
Epochs                  & 2000                                                  \\
Batch size              & 16384                                                     \\
Dropout                 & 0                                                          \\
Policy output           & Unimodal Gaussian           \\\bottomrule                         
\end{tabular}
\end{table}

\begin{table}[ht]
\caption{Hyperparameters in Condion Function.}
\centering
\label{tab:hyper_mlp}
\begin{tabular}{cc}
\toprule
\textbf{Hyperparameter} & \textbf{Value}                                         \\ \midrule
Hidden layers           & 3                                                     \\
Layer width             & 128                                                              \\
Nonlinearity            & ReLU                                                             \\
Learning rate           & $1 \times 10^{-4}$                                            \\
Epochs                  & 300                                                  \\
Batch size              & 256                                                     \\
Dropout                 & 0.1                                                          \\ \bottomrule       
\end{tabular}
\end{table}

\paragraph{Implementation of DT-$\algname$} We also evaluate our proposed $\algname$ method using the DT \cite{chen2021decision} on the D4RL benchmark \citep{fu2020d4rl}. 
DT utilize the transformer model to learn the action given the past trajectory and the target RTG.
For the implementation, we follow the recently proposed Reinformer \citep{zhuang2024reinformer} utilizing additional action entropy to regularize to have a more stable DT training. 
We formulate the loss function as $L_{DT-\algname}=\mathbb{E}_{\tau}[-\log\pi_{\theta}(\cdot|\tau)-\lambda H(\pi_{\theta}(\cdot|\tau))]$. In the vanilla DT, we also need to give the target RTG during the inference stage. However, similar to the implementation on the RVS framework, our method also does not require additional target RTGs to achieve good performance.

We use the same MLP condition function hyperparameters in \Cref{tab:hyper_mlp}. And we provide the hyperparameters of our DT-$\algname$ in \Cref{tab:hyper_DT}.

We run these experiments on the A5000 for around 3 hours for each setting.

\begin{table}[ht]
\caption{Hyperparameters in  DT-$\algname$.}
\centering
\label{tab:hyper_DT}
\begin{tabular}{cc}
\toprule
\textbf{Hyperparameter}          & \textbf{Value}                                                   \\ \midrule
Number of layers                & 3                                                                \\
Number of attention heads       & 1                                                                \\
Embedding dimension             & 128                                                              \\
Nonlinearity function           & ReLU                                                             \\
Batch size                      & 64                                                               \\
Context length \(K\)            & 20                  \\
Dropout                         & 0.1                                                              \\
Learning rate                   & \(1\times10^{-4}\)                                               \\
Grad norm clip                  & 0.25                                                             \\
Weight decay                    & \(1\times10^{-4}\)                                               \\
Action Entropy $\lambda$        & 0.1
\\ \bottomrule
\end{tabular}
\end{table}

\paragraph{Implementation of DP-$\algname$} We also extend our empirical study into the RCSL with the dynamic programming component. We evaluate our proposed DP-$\algname$ method using the QT framework \cite{hu2024q} on the D4RL benchmark \citep{fu2020d4rl}. QT incorporates the Q value function learning in the training stage and utilizes this function to guide the action selection in the inference stage.
An MLP-based conditioning function is used to model the distribution of RTG values conditioned on the input state. Initial target RTGs are then sampled from this distribution and combined with several given target RTGs during the evaluation process. At each inference step, the action is generated given the predicted RTG and then selected by the critic component which is the Q value function.

The hyperparameters for the MLP-based QT model are provided in \Cref{tab:hyper_QT_mlp}, and the training hyperparameters for DP-$\algname$ are summarized in \Cref{tab:hyper_QT}.

We run these DP-$\algname$ experiments on the A5000 for around 15 hours for each setting.

\begin{table}[ht]
\caption{Hyperparameters in Condion Function in DP-$\algname$.}
\centering
\label{tab:hyper_QT_mlp}
\begin{tabular}{cc}
\toprule
\textbf{Hyperparameter} & \textbf{Value}                                         \\ \midrule
Hidden layers           & 3                                                     \\
Layer width             & 256                                                              \\
Nonlinearity            & ReLU                                                             \\
Learning rate           & $1 \times 10^{-4}$                                            \\
Epochs                  & 300                                                  \\
Batch size              & 256                                                     \\
Dropout                 & 0.1                                                          \\ \bottomrule       
\end{tabular}
\end{table}

\begin{table}[ht]
\caption{Hyperparameters in  DP-$\algname$.}
\centering
\label{tab:hyper_QT}
\begin{tabular}{cc}
\toprule
\textbf{Hyperparameter}          & \textbf{Value}                                                   \\ \midrule
Number of layers                & 4                                                                \\
Number of attention heads       & 4                                                                \\
Embedding dimension             & 256                                                              \\
Nonlinearity function           & ReLU                                                             \\
Batch size                      & 256                                                               \\
Context length \(K\)            & 20                  \\
Dropout                         & 0.1                                                              \\
Learning rate                   & \(3.0\times10^{-4}\)                                               \\\bottomrule
\end{tabular}
\end{table}

\end{document}